\documentclass{article}
\usepackage[bookmarksnumbered = true, colorlinks = true, linktocpage = true, bookmarksopen = true,linkcolor = black, urlcolor = black]{hyperref}
\usepackage[intlimits]{amsmath}
\usepackage{subcaption}
\usepackage{amssymb}
\usepackage{comment}
\usepackage{amsthm}
\usepackage{rotating}
\usepackage{xcolor}
\usepackage[numbers,sort]{natbib}  
\usepackage{mathtools}
\usepackage[]{algorithm}
\usepackage{algorithmic}
\usepackage{tikz}
\usepackage{authblk}
\usepackage{booktabs}
\usetikzlibrary{fit,shapes.misc}
\usepackage[utf8]{inputenc}
\usepackage[margin =1in]{geometry}
\usepackage{url}
\usepackage[numbers,sort]{natbib}
\usepackage[normalem]{ulem}

\usepackage[titletoc]{appendix}
\usepackage[bottom,flushmargin,hang,multiple]{footmisc}
\usepackage{setspace}

\usepackage{amsmath,amsbsy,amsgen,amscd,amsthm,amsfonts,amssymb,mathrsfs}
\newtheorem{theorem}{Theorem}

\newtheorem{proposition}[theorem]{Proposition}
\newtheorem{corollary}{Corollary}
\newtheorem{assumption}{Assumption}

\newtheorem{remark}{Remark}
\newtheorem{lemma}{Lemma}

\usepackage{xcolor}
\usepackage{url}
\usepackage{algorithm}
\usepackage{algorithmic}
\usepackage{graphicx}
\usepackage{booktabs}
\usepackage{comment}
\usepackage{amsmath,amsbsy,amsgen,amscd,amsthm,amsfonts,amssymb,mathrsfs}

\title{Exploration strategies for Gaussian process bandit optimization\\ with inference regret guarantees}
\author[1]{Hwanwoo Kim}
\author[2]{Jungwoo Hong}
\author[2]{Kyungjae Lee}
\date{\today}
\affil[1]{Department of Mathematical Sciences, University of Texas at Dallas}
\affil[2]{Department of Statistics,
Korea University}

\begin{document}

\maketitle

\begin{abstract}
Existing cumulative regret guarantees for Bayesian optimization (BO) do not necessarily characterize the quality of the algorithm's final output. This gap is critical under noisy evaluations, where the goal is to return a single high-quality recommendation. To address it, we introduce efficient exploration strategies that can be incorporated into any BO procedure by occasionally replacing the acquisition query with an exploration query.  To validate this approach, we bound the Gaussian process posterior standard deviation over the search space via the fill distance of the queries, which measures the largest region left unexplored. We then use this result to derive upper bounds on the inference regret, i.e., the gap between the global optimum and the function value at the algorithm's returned point. Empirically, these exploration strategies improve the final recommendations of existing BO algorithms on synthetic benchmarks, hyperparameter optimization, and a control task.
\end{abstract}

\section{Introduction}\label{SEC:intro}
Bayesian optimization (BO) has emerged as a principled and sample efficient framework for optimizing expensive black-box functions, with broad applications spanning hyperparameter tuning \citep{wu2020practical,turner2021bayesian,kandasamy2020tuning}, searching for neural network architectures \citep{kandasamy2018neural,zhou2019bayesnas,white2021bananas}, drug discovery \citep{bellamy2022batched,colliandre2023bayesian}, neutrino detector design \citep{kim2025efficient} and designing materials \citep{frazier2016bayesian,zhang2020bayesian,lei2021bayesian}. The predominant paradigm couples a Gaussian process (GP) surrogate \citep{williams2006gaussian, stein2012interpolation} with an acquisition function that trades off exploration and exploitation to guide the search. While this framework enjoys strong theoretical guarantees in the noiseless setting, its behavior under noisy evaluations is considerably more subtle: the acquisition function may systematically under-explore, causing the posterior to remain poorly calibrated over large regions of the search space and leading to unreliable final recommendations.

A natural measure of the quality of a BO algorithm's final output is the \emph{inference regret}, defined as $\max_{x} f(x) - f(\hat{x}_T)$, where $\hat{x}_T$ is the algorithm output after $T$ evaluations. Unlike cumulative regret, which penalizes every suboptimal query, inference regret focuses solely on the quality of the final recommendation, which is the operationally relevant quantity in most practical deployments where a single solution is returned to the user. Ensuring vanishing inference regret requires the posterior to contract uniformly over the search space, which in turn demands sufficient exploration throughout the optimization horizon.

Several exploration strategies have been proposed to enforce such coverage. \emph{Maximum variance reduction} (MVR) selects at each round the point of highest posterior uncertainty, provably driving the posterior variance to zero and attaining order-optimal inference regret up to logarithmic factors \citep{vakili2021optimal}. These guarantees come at a price: every round requires a global optimization of the posterior variance, so MVR is as computationally expensive as a standard acquisition step. At the opposite end, \emph{$\epsilon$-greedy} exploration replaces the acquisition query with a uniformly random point with fixed probability. It is simple, computationally free, and empirically highly effective \citep{de2021greed, do2024epsilon}, yet its theoretical understanding remains limited. Existing analyses are confined to the noiseless setting or to a particular acquisition function \citep{bull2011convergence}, and no inference regret guarantees are available when it is paired with a general acquisition function under noisy evaluations. Beyond these two, the exploration rules developed in approximation theory and experimental design, including farthest-point selection and other quasi-uniform constructions that spread points over a domain as evenly as possible \citep{wendland2004scattered, pronzato2023quasi}, have received little attention as exploration mechanisms within BO.

\paragraph{This work.} We study a generic scheme in which, at each $t^{\text{th}}$ round, the BO algorithm follows its acquisition function with probability $1-p_t$ and otherwise spends the evaluation on a query chosen by a separate exploration rule. The scheme is characterized by an exploration schedule $\{p_t\}_{t\in\mathbb{N}}$ and an exploration rule, while imposing no restrictions on the underlying acquisition function. Existing strategies, including $\epsilon$-greedy, the alternating scheme of \citet{kim2025enhancing}, and pure exploration methods such as MVR, arise as special cases of this scheme. To analyze this scheme, we quantify the coverage of the queries by the largest distance from any point of the domain to its nearest query, a quantity known as the \emph{fill distance}. We bound the GP posterior standard deviation uniformly over the search space in terms of the fill distance, which holds for any query set. 

We then consider two exploration rules, greedy packing and random sampling. With high probability, both drive the fill distance to zero at the optimal rate (up to logarithmic factors for random sampling), even when interleaved with queries from an arbitrary acquisition function. We translate these results into inference regret guarantees for the posterior mean maximizer recommendation under noisy observations, with rates expressed explicitly in terms of the exploration budget $\sum_{t=1}^T p_t$, where $T$ is the total number of iterations. In particular, under a constant exploration schedule, random sampling recovers $\epsilon$-greedy and enjoys inference regret guarantees with essentially no additional computational overhead, providing a lightweight safeguard for practical BO systems.

\paragraph{Contributions.} Our main contributions include:
\begin{itemize}
    \item \textbf{Comprehensive exploration scheme.} We formalize a family of exploration strategies for BO, which can be seamlessly incorporated into existing acquisition functions without additional computational overhead.
    
    \item \textbf{Posterior variance bounds.} 
    We establish an upper bound on the Gaussian process posterior variance in terms of fill distance, which quantifies how well the query locations cover the search space. Unlike existing bounds based on maximum information gain, our fill-distance characterization allows a broad class of space-filling designs to be incorporated as exploration strategies in Bayesian optimization and provides a framework for analyzing their performance.

    \item \textbf{Inference regret bounds.} Motivated by the fact that cumulative regret does not, in general, characterize the quality of the final recommendation under noisy evaluations, we establish inference regret bounds for the posterior mean maximizer under a range of exploration strategies. Our results identify conditions under which the final recommendation converges to the optimum. For uniform random exploration, they also extend the noiseless $\epsilon$-greedy expected improvement convergence guarantee of \citet{bull2011convergence} to noisy evaluations and a broader class of acquisition functions.

    \item \textbf{Empirical validation.} We demonstrate the effectiveness of the proposed exploration strategies on benchmark functions, hyperparameter tuning, and control tasks, showing in particular that uniform random exploration provides a computationally efficient way to complement standard BO and consistently improves inference regret. 
\end{itemize}

\section{Background}
\label{sec:background}
\subsection{Problem Statement}
We seek the global maximizer $x^\star$ of an objective function $f:\mathcal{X}
:=
\prod_{j=1}^d[a_j,b_j]
\subset\mathbb R^d, \to \mathbb{R}$ whose functional form is unknown. Information about $f$ is available only through noisy evaluations, which may be expensive to obtain; the goal is therefore to locate the optimum with as few queries as possible. We denote the query locations up to time $t \in \mathbb{N}$ by $X_t = \{x_1, \ldots, x_t\}$ and the corresponding noisy observations by $Y_t = [f(x_1) + \epsilon_1, \ldots, f(x_t)+\epsilon_t]^\top$, where the noise terms $\{\epsilon_i\}_
{i \in \mathbb{N}}$ are i.i.d. $R$-sub-Gaussian random variables, i.e., $\mathbb{E}\left[e^{\gamma \epsilon_i}\right] \le \exp\left(\frac{\gamma^2 R^2}{2}\right)$ for all $\gamma \in \mathbb{R}$.

\subsection{Gaussian Processes}
A Gaussian process (GP) provides a flexible surrogate model for an unknown objective
function $f:\mathcal{X}\to\mathbb{R}$ and is characterized by a mean function
$\mu:\mathcal{X}\to\mathbb{R}$ and a symmetric positive-definite kernel
$k:\mathcal{X}\times\mathcal{X}\to\mathbb{R}_{\ge 0}$. We do not assume that the objective function is itself generated from a GP; throughout, the GP is used as a surrogate for optimization. Given observations
$\mathcal{D}_t=\{(x_i,y_i)\}_{i=1}^t$, let
$
{k}_t(x)
=
[k(x_1,x),\ldots,k(x_t,x)]^\top,
~
{K}_{tt}
=
[k(x_i,x_j)]_{i,j=1}^t.
$
We define the surrogate mean and variance by
\begin{align}
\mu_t(x)
&=
k_t(x)^\top
(K_{tt}+\lambda I)^{-1}Y_t,
\label{pos_mean}\\
\sigma_t^2(x)
&=
k(x,x)
-
k_t(x)^\top
(K_{tt}+\lambda I)^{-1}
k_t(x),
\label{pos_var}
\end{align}
where $\lambda>0$ is a regularization parameter, often referred to as the
nugget. These expressions coincide with the usual GP posterior mean and
variance under Gaussian observation noise, but remain well defined under the
sub-Gaussian noise assumption considered here. We therefore retain the
standard terminology of posterior mean and posterior variance when referring
to $\mu_t$ and $\sigma_t^2$.
The posterior mean $\mu_t$ provides the estimate of $f$ given $\mathcal{D}_t$, while the posterior variance $\sigma_t^2(x)$ quantifies epistemic uncertainty: it is large in unexplored regions and contracts toward zero as observations accumulate.

The choice of kernel encodes prior beliefs about the smoothness of $f$. The squared-exponential (SE) kernel, 
$$ 
k_{\mathrm{SE}}(x, x') = \exp\left(-\frac{\|x - x'\|^2}{2\ell^2}\right), 
$$ 
with lengthscale $\ell > 0$, produces sample paths that are almost surely smooth to all orders \citep{williams2006gaussian, kanagawa2018gaussian}, suitable for modeling functions with higher-order differentiability. The Matérn kernel \citep{stein2012interpolation}, 
{\small
\[
k_{\nu}(x, x') = \frac{2^{1-\nu}}{\Gamma(\nu)} \left(\frac{\sqrt{2\nu}\,\|x-x'\|}{\ell}\right)^{\nu} K_{\nu}\!\left(\frac{\sqrt{2\nu}\,\|x-x'\|}{\ell}\right),
\]
}where $K_\nu$ is the modified Bessel function of the second kind, generalizes the SE kernel to a broader class of functions whose smoothness is controlled by the parameter $\nu > 0$.

\subsection{Bayesian Optimization}
Bayesian optimization (BO) aims to locate the global maximizer $x^\star$ of $f$ using as few noisy evaluations as possible, by maintaining a GP surrogate of $f$ and using it to decide where to evaluate next. At each round $t \in \mathbb{N}$, BO selects the next query by maximizing an acquisition function $\alpha_t: \mathcal{X} \to \mathbb{R}$, i.e., $x_t = \arg\max_{x \in \mathcal{X}} \alpha_t(x)$, then observes 
$y_t = f(x_t) + \epsilon_t$ and updates the surrogate model. The procedure is summarized in Algorithm~\ref{alg:gpbo}.

\begin{algorithm}[htp]
\caption{Bayesian optimization}
\label{alg:gpbo}
\begin{algorithmic}[1]
\REQUIRE Kernel $k$, initial data $\mathcal{D}_0$
\FOR{$t = 1, 2, \dots$}
  \STATE Compute \eqref{pos_mean} and \eqref{pos_var} from $\mathcal{D}_{t-1}$
  \STATE Select $x_t = \arg\max_{x \in \mathcal{X}} \alpha_t(x)$
  \STATE Observe $y_t = f(x_t) + \epsilon_t$ 
  \STATE Update $\mathcal{D}_t = \mathcal{D}_{t-1} \cup \{(x_t, y_t)\}$
\ENDFOR
\end{algorithmic}
\end{algorithm}
Designing $\alpha_t$ involves two competing goals: \emph{exploitation}, which favors points the surrogate currently predicts to be good, and 
\emph{exploration}, which favors points where the surrogate remains uncertain. An effective acquisition function trades off the two, refining regions that already look promising while continuing to probe regions whose potential is still unknown. Acquisition functions differ in how they strike this balance; widely used examples include expected improvement 
\citep{mockus1978application, jones1998efficient, ament2023unexpected}, upper confidence bound \citep{srinivas2010gaussian}, Thompson sampling \citep{russo2014learning, chowdhury2017kernelized}, and entropy-based methods 
\citep{hennig2012entropy, wang2017max}.

\subsection{Evaluation Metric}
\label{subsec:Eval_Metric}
A widely used performance measure for BO is cumulative regret given by, 
$$ 
R_T := \sum_{t=1}^T r_t, \quad r_t := f(x^*) - f(x_t) \geq 0, 
$$ 
where the term $r_t$ is known as the instantaneous regret. A no-regret algorithm satisfies $R_T/T \to 0$ as $T \to \infty$, meaning the average quality of queries converges to the optimum. The related simple regret $S_T := \min_{t = 1,\ldots, T} r_t \leq R_T/T$ measures the quality of the single best point found. Establishing sublinear growth of cumulative regret implies convergence of simple regret to zero, and has served as the primary theoretical validation for BO strategies across a wide range of assumptions on the objective \citep{srinivas2010gaussian, russo2014learning, chowdhury2017kernelized, whitehouse2023sublinear, kim2024optimization, kim2025inexact, iwazaki2025improved, takeno2026regret}.

\subsection{Inference Regret: A Gap in Existing Theory}
Despite its usefulness, the cumulative regret perspective has two important limitations for assessing the final outcome of BO. First, the implication $S_T \leq R_T/T$ can be coarse: even a bounded cumulative regret guarantee yields at most an $\mathcal{O}(T^{-1})$ simple-regret rate and therefore cannot capture potentially faster convergence of the best query. Second, in the noisy evaluation setting, the best queried point is not discernable. Sublinear cumulative regret ensures that some queries become nearly optimal, but does not indicate which queried point should ultimately be returned.

This distinction is not problematic under noiseless evaluations, where the best queried point can be identified directly from the observed function values. Under noisy evaluations, however, the algorithm must infer its final recommendation from noisy data. We therefore consider the inference regret \citep{wang2017max},
$
I_T := f(x^\star)-f(\hat{x}_T),
$
where $\hat{x}_T$ denotes the algorithm's final recommendation. A common choice in noisy BO is the maximizer of the final posterior mean,
$
\hat{x}_T = \arg\max_{x\in\mathcal X}\mu_T(x).
$
While simple regret describes the quality of the best point queried in hindsight, inference regret measures the quality of the point that the algorithm actually recommends. The two therefore need not behave similarly in the presence of noise. Although these notions are sometimes conflated in the literature \citep{garnett2023bayesian}, we maintain this distinction throughout.

Figure~\ref{fig:motivation} illustrates the limitation of cumulative regret. The query strategy exhibits favorable sublinear cumulative regret behavior and increasingly concentrates its evaluations near the true maximizer $x^\star$. Nevertheless, observation noise can lead to a poor posterior mean approximation in sparsely explored regions, and insufficient subsequent exploration may prevent these errors from being corrected, allowing a spurious posterior mean maximum to persist. As a result, the posterior mean recommendation remains separated from $x^\star$ and incurs non-vanishing inference regret even though the query sequence itself is well behaved. Thus, favorable cumulative regret performance need not translate into a final recommendation that converges to the true optimum.

This issue is especially important in noisy BO, where the practitioner typically returns a single recommendation. Yet inference regret has received comparatively limited theoretical attention \citep{vakili2021optimal,scarlett2017lower}, and standard cumulative regret guarantees do not, in general, ensure convergence of a recommendation such as $\arg\max_{x \in \mathcal{X}} \mu_T(x)$. The example in Figure~\ref{fig:motivation} highlights what is missing: good query performance alone is insufficient, and uncertainty must also be reduced across the search space. Otherwise, posterior mean approximation errors in sparsely explored regions may persist and distort the final recommendation.
\begin{figure*}[htp]
\centering
\includegraphics[width=1\textwidth]{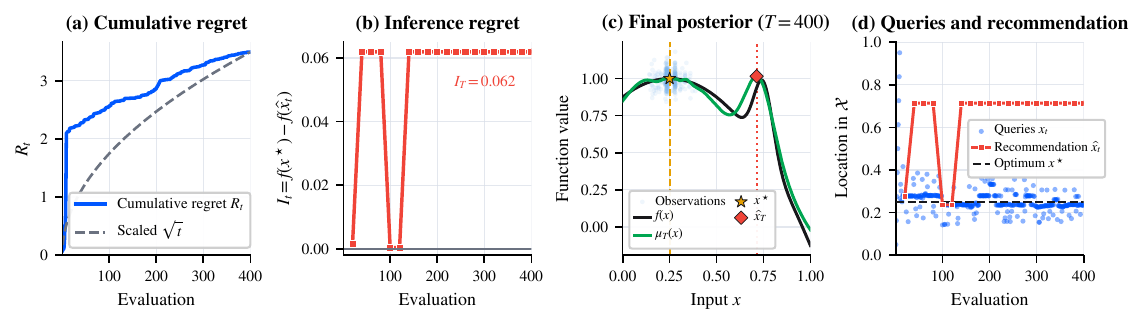}
\caption{Cumulative regret and inference regret under noisy Bayesian optimization. The query strategy achieves favorable cumulative regret behavior and concentrates evaluations near the true maximizer, yet the posterior mean recommendation remains suboptimal. From the user's perspective, however, the noisy evaluations do not reveal which queried points are actually close to the true maximizer. This illustrates that a sublinear cumulative regret guarantee does not necessarily imply vanishing inference regret.}
\label{fig:motivation}
\end{figure*}

\section{Exploration Schemes for Bayesian Optimization}
\label{sec:PROB_EXPLORE}
A reliable final recommendation requires the surrogate to be accurate throughout the search space, not only in regions frequently queried by the algorithm. Acquisition functions alone may not provide such a guarantee: once a region appears unpromising, it may receive little further evaluation. Motivated by this limitation, we introduce a family of exploration schemes that serve as a generic wrapper around an arbitrary acquisition function. The additional exploration step promotes global coverage of the search space while leaving the underlying acquisition rule unchanged. We show that this coverage yields uniform control of posterior uncertainty and, in turn, inference regret guarantees.

\subsection{A Generic Exploration Scheme}\label{subsec:framework}
We propose a BO strategy that, at each round, either follows the acquisition function as usual or spends the evaluation on a query chosen by a separate exploration strategy. Formally, given an exploration schedule $\{p_t\}_{t\in\mathbb{N}}\subset[0,1]$, let $\{E_t\}_{t\in\mathbb{N}}$ be independent Bernoulli random variables with $E_t\sim\mathrm{Bernoulli}(p_t)$. The query at round $t\in\mathbb{N}$ is
\begin{equation}\label{EQN:PROB_EXP}
x_t =
\begin{cases}
x_t^{\mathrm{exp}}, & E_t=1,\\
\displaystyle \arg\max_{x\in\mathcal{X}} \alpha_t(x), & E_t=0,
\end{cases}
\end{equation}
where $\alpha_t$ is an acquisition function of interest and $x_t^{\mathrm{exp}}$ is specified by an exploration rule designed to reduce surrogate uncertainty. The scheme has two ingredients, the exploration schedule and the exploration rule, and places no restriction on the acquisition function.

Several existing strategies are special cases of \eqref{EQN:PROB_EXP}. Setting $p_t = 0$ for all $t \in \mathbb{N}$ recovers vanilla BO with an acquisition rule $\alpha_t$, while setting $p_t = 1$ for all $t \in \mathbb{N}$ with $x_t^{\mathrm{exp}} = \arg\max_{x \in \mathcal{X}} \sigma_{t-1}^2(x)$ recovers the MVR strategy of \citet{vakili2021optimal}. A constant schedule $p_t = \epsilon$ for all $t \in \mathbb{N}$ with uniformly random queries gives the $\epsilon$-greedy strategy \citep{bull2011convergence, de2021greed}, and alternating between exploration and acquisition rounds, i.e., $p_t = 1$ for even $t$ and $p_t = 0$ for odd $t$, recovers the exploration scheme of \citet{kim2025enhancing}.

\subsection{Posterior Variance Bound}\label{subsec:variance-decay}

The proposed BO strategy \eqref{EQN:PROB_EXP} imposes no restriction on the underlying acquisition function. We therefore seek a posterior variance bound that depends only on the resulting query locations. This is natural because the posterior variance is independent of the observed function values and decreases as the query locations become denser throughout the search space. We measure how densely a set $X_n = \{x_1, \dots, x_n\} \subset \mathcal{X}$ covers the domain by its \emph{fill distance}
\begin{equation}\label{EQN:FILL}
h_{X_n} := \sup_{x \in \mathcal{X}} \min_{1 \le i \le n} \|x - x_i\|_2,
\end{equation}
the largest distance from a point of $\mathcal{X}$ to its nearest query point. Relating the posterior variance to $h_{X_n}$ requires a smoothness condition on the kernel.

\begin{assumption}\label{ASSUMP:KERNEL}
For some $\tau > d/2$, the reproducing kernel Hilbert space $\mathcal{H}_k(\mathcal{X})$ associated with the kernel $k$ is continuously embedded in the Sobolev space $W_2^\tau(\mathcal{X}^\circ)$, where $\mathcal{X}^\circ$ denotes the interior of $\mathcal{X}$. That is, $\|g|_{\mathcal{X}^\circ}\|_{W_2^\tau(\mathcal{X}^\circ)} \lesssim \|g\|_{\mathcal{H}_k(\mathcal{X})}$ for all $g \in \mathcal{H}_k(\mathcal{X})$, where $g|_{\mathcal{X}^\circ}$ denotes the restriction of $g$ to $\mathcal{X}^\circ$.
\end{assumption}

Assumption~\ref{ASSUMP:KERNEL} holds for the Mat\'ern kernel with $\tau=\nu+d/2$ \citep{wendland2004scattered}, and for the SE kernel for every integer $\tau>d/2$ \citep{steinwart2008svm}. By combining a sampling inequality \citep{arcangeli2012extension} with a Gagliardo--Nirenberg inequality \citep{brezis2019sobolev}, we obtain the following posterior variance bound in terms of the coverage of the query locations over the search space.

\begin{theorem}[Posterior variance and fill distance]\label{THM:PV_FILL}
Under Assumption~\ref{ASSUMP:KERNEL}, there exists a constant $h_0 > 0$ such that every query set $X_n \subset \mathcal{X}$ with $h_{X_n} \le h_0$ satisfies
\[
\sup_{x \in \mathcal{X}} \sigma_n^2(x) = \mathcal{O}\left(h_{X_n}^{2\tau - d} + \lambda^{1 - \frac{d}{2\tau}}\, h_{X_n}^{d\left(1 - \frac{d}{2\tau}\right)}\right),
\]
where the constant in $\mathcal{O}(\cdot)$ is independent of $n$.
\end{theorem}

The two terms inside the big-$\mathcal{O}$ notation play distinct roles. The leading term reflects the error rate of noiseless kernel interpolation \citep{wendland2004scattered}: it reflects how well the queries determine a function of smoothness $\tau$ and is the only term that remains as $\lambda \to 0$. The remainder term is the price of observation noise: it vanishes when $\lambda = 0$, but its exponent in $h_{X_n}$ is smaller than that of the leading term, so it decays more slowly and dominates the bound as the queries fill the domain whenever $\lambda > 0$.

Theorem~\ref{THM:PV_FILL} holds for any set of queries, including those generated by \eqref{EQN:PROB_EXP}. It therefore reduces the problem of bounding the posterior variance over the entire search space to a geometric one: how well do the queries cover \(\mathcal X\)? The most efficient coverage is achieved by point sets that remain well separated while leaving no large gaps in the domain. This is precisely the defining feature of quasi-uniform point sets. A family of point sets \(X_n\) is \emph{quasi-uniform} if its fill distance is comparable to its separation distance
$
q_{X_n}:=\frac{1}{2}\min_{i\neq j}\|x_i-x_j\|_2,
$
that is, if \(h_{X_n}\le c\,q_{X_n}\) for some constant \(c>0\) \citep{wendland2004scattered}. Such point sets satisfy \(h_{X_n}\asymp n^{-1/d}\), which is the fastest possible decay rate.


\subsection{Exploration Strategies}
Theorem~\ref{THM:PV_FILL} shows that reducing the fill distance improves the uniform posterior variance bound. This motivates exploration strategies that spread queries effectively across the search space. We consider two such strategies, greedy packing and random sampling, and analyze their fill-distance rates when the BO procedure \eqref{EQN:PROB_EXP} selects these exploration steps with probability \(p_t\) at round \(t\). Let \(M_T\) denote the number of exploration rounds among the first \(T\) iterations, and define
$ e_T:=\mathbb{E}[M_T]=\sum_{t=1}^T p_t, $
which we refer to as the \emph{exploration budget}.

\paragraph{Greedy packing.}
This strategy selects a point farthest from all previous
queries:
\[
x_t^{\mathrm{exp}}
\in
\arg\max_{x\in\mathcal{X}}
\min_{s<t}\|x-x_s\|_2.
\]
If no queries have yet been made, the first point is chosen arbitrarily. When every round explores, this reduces to the classical farthest-point construction, which generates quasi-uniform designs \citep{pronzato2023quasi}. The next result shows that this quasi-uniform coverage is preserved with high probability, up to the effective number of exploration rounds, even when arbitrary acquisition queries are interleaved.

\begin{proposition}[Fill distance under greedy packing]
\label{PROP:FILL_GREEDY}
Suppose that the exploration schedule satisfies
$e_T=\sum_{t=1}^T p_t\to\infty$.
Under greedy packing, for every $\delta\in(0,1)$, with probability at least $1-\delta$,
\[
h_{X_T} = \mathcal{O}\bigl(e_T^{-1/d}\bigr).
\]
\end{proposition}

\paragraph{Random sampling.}
This strategy draws each exploration point independently from a distribution with a positive density on \(\mathcal{X}\), such as the uniform distribution. The draws may be generated either in advance or sequentially as needed. In either case, they are independent of the exploration decisions and of all other algorithmic randomness. Unlike greedy packing, random sampling requires no optimization. Nevertheless, it achieves the same fill-distance rate as greedy packing up to logarithmic factors \citep[Lemma~2]{oates2019convergence}. The following result shows that this near-optimal coverage is retained with high probability when random-sampling exploration is interleaved with acquisition queries.

\begin{proposition}[Fill distance under random sampling]
\label{PROP:FILL_RANDOM}
Suppose that the exploration schedule satisfies
$e_T=\sum_{t=1}^T p_t\to\infty$.
Under random sampling from a distribution whose density is positive on $\mathcal{X}$,
for every $\delta\in(0,1)$,
with probability at least $1-\delta$,
\[
h_{X_T}=\tilde{\mathcal{O}}(e_T^{-1/d}),
\]
where $\tilde{\mathcal{O}}$ denotes $\mathcal{O}$ up to polylogarithmic order terms.
\end{proposition}

\begin{remark}
Although we focus on greedy packing and random sampling, the fill-distance formulation of Theorem~\ref{THM:PV_FILL} allows other exploration strategies to be incorporated directly. For instance, under suitable regularity conditions, $\gamma$-stabilized greedy algorithms generate quasi-uniform point sets \citep{wenzel2021novel} and hence satisfy
$
h_{X_T}=\mathcal{O}(e_T^{-1/d}).
$
They therefore yield the same posterior variance rate as greedy packing. 
\end{remark}

Combining Propositions~\ref{PROP:FILL_GREEDY} and \ref{PROP:FILL_RANDOM} with Theorem~\ref{THM:PV_FILL} yields uniform bounds on the GP posterior variance.

\begin{corollary}[Posterior variance under exploration]
\label{COR:PV_RATES}
Under Assumption~\ref{ASSUMP:KERNEL} and the conditions of
Propositions~\ref{PROP:FILL_GREEDY} and~\ref{PROP:FILL_RANDOM}, for every
$\delta\in(0,1)$, with probability at least $1-\delta$,
\[
\sup_{x\in\mathcal{X}}\sigma_T^2(x)
=
\begin{cases}
\mathcal{O}\bigl(e_T^{-\frac{2\nu}{2\nu+d}}\bigr),
& \text{Mat\'ern, greedy},\\
\tilde{\mathcal{O}}\bigl(e_T^{-\frac{2\nu}{2\nu+d}}\bigr),
& \text{Mat\'ern, random}.
\end{cases}
\]
For the SE kernel, for any arbitrarily small \(\varepsilon\in(0,1)\), either exploration strategy satisfies, with probability at least \(1-\delta\),
\[
\sup_{x\in\mathcal{X}}\sigma_T^2(x)
=
\mathcal{O}\bigl(e_T^{-(1-\varepsilon)}\bigr).
\]
\end{corollary}

Corollary~\ref{COR:PV_RATES} expresses posterior variance decay in terms of
the exploration budget $e_T$. For a concrete exploration schedule, these
rates can be translated directly into rates in the total number of BO
iterations.

\begin{remark}[Explicit rates]
\label{REM:EXPLICIT_RATES}
Consider the exploration schedule
$
p_t=p_0t^{-\alpha},~
p_0\in(0,1],
~
\alpha\in[0,1).
$
Since $e_T\asymp T^{1-\alpha}$,
Corollary~\ref{COR:PV_RATES} gives, for every fixed
$\delta\in(0,1)$, with probability at least $1-\delta$,
\[
\sup_{x\in\mathcal{X}}\sigma_T^2(x)
=
\begin{cases}
\mathcal{O}\bigl(T^{-(1-\alpha)\frac{2\nu}{2\nu+d}}\bigr),
& \text{Mat\'ern, greedy},\\
\tilde{\mathcal{O}}\bigl(
T^{-(1-\alpha)\frac{2\nu}{2\nu+d}}\bigr),
& \text{Mat\'ern, random},\\
\mathcal{O}\bigl(T^{-(1-\alpha)(1-\varepsilon)}\bigr),
& \text{SE, either strategy},
\end{cases}
\]
for an arbitrarily small $\varepsilon\in(0,1)$. A smaller $\alpha$ allocates
more rounds to exploration, resulting in faster posterior variance decay.
\end{remark}

\begin{remark}[Comparison with existing posterior variance bound]
\label{REM:PV_COMPARISON}
When $\alpha=0$, our rates are comparable to existing posterior variance bounds obtained via maximum information gain based analysis. In particular, pure random exploration achieves
$
\tilde{\mathcal O}\left(T^{-2\nu/(2\nu+d)}\right)
$
for Matérn kernels \citep{salgia2023random}, while the MVR strategy achieves the same rate \citep{vakili2021optimal}. With our greedy packing strategy, we further improve this rate by eliminating the polylogarithmic factors.
\end{remark}

\begin{remark}[Computational cost versus statistical efficiency]
\label{REM:RANDOM_EFFICIENCY}
Random exploration achieves a rate comparable to that of greedy packing, up to a mild additional factor. Computationally, however, it is considerably simpler: each exploration point is obtained through a random draw and requires no optimization over $\mathcal{X}$. It therefore offers an efficient alternative to optimization-based exploration strategies such as greedy packing, $\gamma$-stabilized greedy algorithms, and MVR, while retaining nearly the same posterior variance decay rate.
\end{remark}

\section{Inference Regret Bounds}
In this subsection, using the previously established posterior variance bounds, we establish inference regret bounds for the recommendation
$
\hat{x}_T = \arg\max_{x \in \mathcal{X}} \mu_T(x):
$
the maximizer of the final posterior mean. This choice of final output is a popular recommendation rule in noisy Bayesian optimization, as it aggregates all observed information through the posterior. The following theorem establishes the inference regret bound for this recommendation using the confidence bound of \citet{flynn2025tighter}.

\begin{theorem}[Inference regret under probabilistic exploration]
\label{THM:INFERENCE_GENERAL}
Suppose Assumption~\ref{ASSUMP:KERNEL} holds with a fixed nugget $\lambda > 0$ and \(f\in\mathcal H_k(\mathcal X)\) satisfies
$
\|f\|_{\mathcal H_k}\le B
$ for some $B \in \mathbb{R}_{>0}$. Consider the exploration scheme \eqref{EQN:PROB_EXP} with
either greedy packing or random sampling and
$
e_T=\sum_{t=1}^T p_t\longrightarrow\infty.
$
Then, for every $\delta\in(0,1)$, with probability at least $1-\delta$,
\[
I_T
=
\begin{cases}
\displaystyle
\tilde{\mathcal O}\left(
T^{\frac{d}{2(2\nu+d)}}
e_T^{-\frac{\nu}{2\nu+d}}
\right),
& \text{Mat\'ern},\\
\displaystyle
\tilde{\mathcal O}\bigl(
e_T^{-\frac{1-\varepsilon}{2}}
\bigr),
& \text{SE},
\end{cases}
\]
for an arbitrarily small $\varepsilon\in(0,1)$.
\end{theorem}

\begin{remark}[Polynomially decaying exploration]
\label{REM:POLY_EXPLORATION}
Consider the exploration schedule
$
p_t=p_0t^{-\alpha},~
p_0\in(0,1],
~
\alpha\in[0,1).
$
Since
$
e_T\asymp T^{1-\alpha},
$
Theorem~\ref{THM:INFERENCE_GENERAL} gives
\[
I_T
=
\begin{cases}
\displaystyle
\tilde{\mathcal O}\left(
T^{-\frac{2\nu(1-\alpha)-d}
        {2(2\nu+d)}}
\right),
& \text{Mat\'ern},\\
\displaystyle
\tilde{\mathcal O}\left(
T^{-\frac{(1-\alpha)(1-\varepsilon)}{2}}
\right),
& \text{SE}.
\end{cases}
\]
Thus, with a fixed nugget $\lambda > 0$, the Mat\'ern bound vanishes whenever
$
2\nu(1-\alpha)>d,
$
while the SE bound vanishes for every $\alpha\in[0,1)$. 
\end{remark}

\begin{remark}[Improved rate with varying regularization]
\label{REM:VARYING_LAMBDA}
The bound in Theorem~\ref{THM:INFERENCE_GENERAL} can be improved by appropriately choosing the nugget  \(\lambda\). In particular, for the
Mat\'ern kernel, retaining the dependence on $\lambda$ in both the
posterior variance and the adaptive confidence bound gives
\[
I_T
=
\tilde{\mathcal O}\left(
e_T^{-\frac{\nu}{2\nu+d}}
\left(
\lambda^{\frac{\nu}{2\nu+d}}
+
T^{\frac{d}{2(2\nu+d)}}
\lambda^{-\frac{d}{2\nu+d}}
\right)
\right).
\]
Balancing the two terms, we can set
$
\lambda
\asymp
T^{\frac{d}{2\nu+2d}},
$
which leads to
$
I_T
=
\tilde{\mathcal O}\left(
T^{\frac{d\nu}
        {2(\nu+d)(2\nu+d)}}
e_T^{-\frac{\nu}{2\nu+d}}
\right).
$
Hence, under the polynomially decaying exploration schedule
$p_t=p_0t^{-\alpha}$,
$
I_T
=
\tilde{\mathcal O}\left(
T^{-\frac{\nu}{2\nu+2d}
+\frac{\alpha\nu}{2\nu+d}}
\right),
$
for $\alpha \in [0,1)$. This improves upon the condition in Theorem~\ref{THM:INFERENCE_GENERAL}. In particular, the bound vanishes whenever
$
\alpha < \frac{2\nu+d}{2\nu+2d}.
$
\end{remark}

\begin{remark}[Improved rate with non adaptive design] 
A sharper inference regret bound can be obtained when every round is devoted to exploration, by leveraging the tighter confidence bound in \citet{vakili2021optimal}.
\end{remark}

\section{Numerical Experiments}
\label{sec:experiments}
We evaluate the proposed exploration-aided BO strategies on benchmark test functions, hyperparameter optimization (HPO), and Lunar Lander control. Complete experimental details and additional results are provided in Appendix~\ref{APPEND:ADD_EXP}. A code repository for reproducing the numerical experiments in this section is available at \url{https://github.com/hsleejw21/probabilistic-exploration}. 

\begin{figure*}[!t]
\centering
\includegraphics[width=\textwidth]{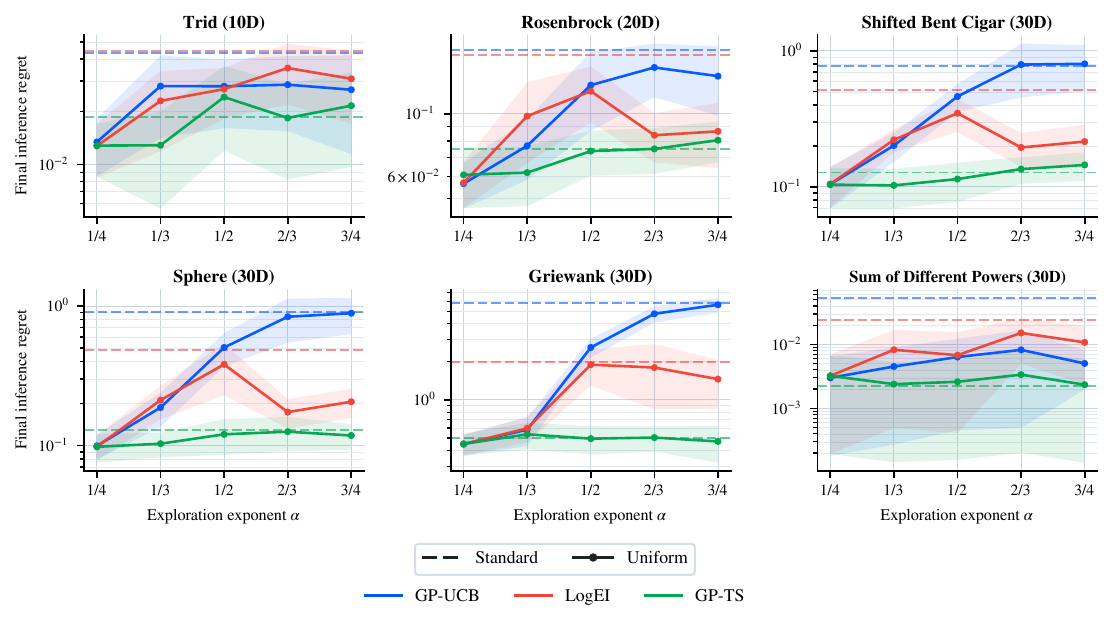}
\caption{\textbf{Schedule sensitivity of uniform random exploration.} Average final inference regret over 20 independent runs across six synthetic benchmarks, with shaded bands representing one standard deviation. Dashed horizontal lines denote standard BO, while solid curves denote BO with uniform random exploration. Lower values indicate better performance.}
\label{fig:exp_synthetic_alpha_sweep}
\end{figure*}

\paragraph{Synthetic benchmark functions.}
Figure~\ref{fig:exp_synthetic_alpha_sweep} examines the effectiveness of combining uniform random exploration with popular BO acquisition functions on six benchmark objectives \citep{simulationlib,Al-Roomi2015}, with dimensions ranging from 10 to 30. Each curve reports the mean final inference regret over 20 independent runs. Over a broad range of \(\alpha\), complementing standard BO with random exploration generally reduces final inference regret, with particularly consistent improvements for GP-UCB and LogEI \citep{ament2023unexpected}. The results also suggest that greater exploration can be beneficial in higher dimensions, although the preferred schedule depends on both the objective and acquisition function. In particular, \(\alpha=1/2\) yields lower average final inference regret in 16 of 18 scenarios, including every GP-UCB and LogEI case. In Appendix~\ref{app:exp_synthetic}, we plot inference regret trajectories for both the common choice \(\alpha=1/2\) and a calibrated choice of \(\alpha\), and report final inference regret across all objective--acquisition pairs.

\begin{figure*}[!t]
\centering
\includegraphics[width=\textwidth]{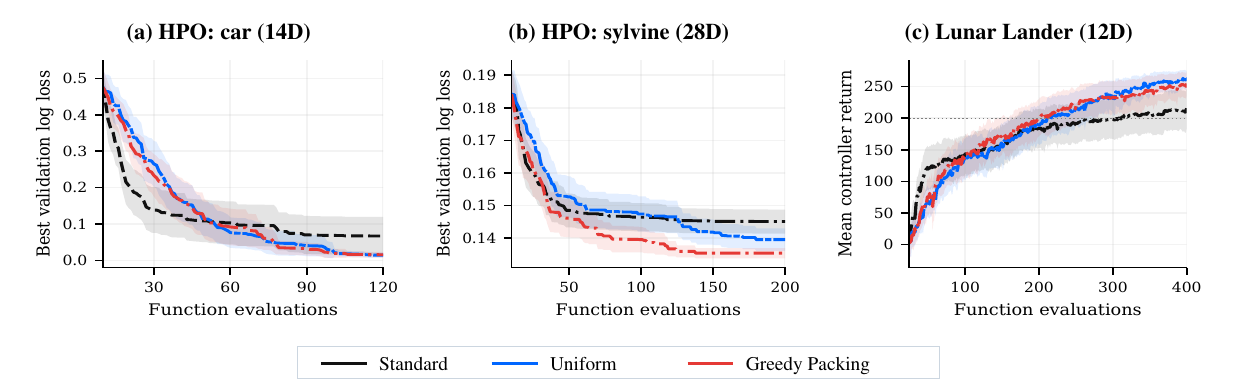}
\caption{\textbf{HPO and control trajectories.} All methods use LogEI; both exploration rules use $\alpha=1/2$. Means are shown with pointwise \(1.96\)-standard error bands over 30 independent runs. Panels use separate vertical scales: (a,b) best validation log loss (lower is better); (c) recommended controller return on 50 terrains fixed within each optimization run (higher is better). The dotted line marks the threshold of mean return 200.}
\label{fig:exp_applications}
\end{figure*}

\paragraph{Hyperparameter optimization and control.}
YAHPO Gym \citep{pfisterer2022yahpo} provides surrogate benchmarks for evaluating HPO methods without repeated model training. Figure~\ref{fig:exp_applications}(a,b) shows the 14D \texttt{car} and 28D \texttt{sylvine} tasks. In both cases, augmenting BO with either random exploration or greedy packing improves the final validation loss over Standard BO. In the experiments, we replace exact greedy packing with an approximate greedy-packing rule over a finite candidate set, providing a computationally efficient implementation that avoids solving a continuous optimization problem at each exploration step \citep{pronzato2023quasi}. Additional HPO results and detailed experimental settings are provided in Appendix~\ref{app:exp_hpo}. Following the Lunar Lander controller-tuning setup of \citet{eriksson2019scalable}, we additionally use BO to optimize a 12-parameter heuristic controller over 400 evaluations. Figure~\ref{fig:exp_applications}(c) shows that both exploration-aided strategies improve the recommended controller over Standard BO, with higher mean returns. Experimental details and a visual illustration of the Lunar Lander controller performance are provided in Appendix~\ref{app:exp_lunar}.

\paragraph{Computational efficiency.}
On two-dimensional benchmark functions, we compare standard BO with variants that augment the base acquisition rule using either uniform random exploration or MVR-based exploration. BO with uniform random exploration achieves the lowest mean final inference regret while reducing average runtime by roughly 40\% relative to both standard BO and MVR-based exploration. These results highlight the computational efficiency of incorporating a simple random sampling-based exploration strategy into BO. Experimental details and numerical results are provided in Appendix~\ref{app:exp_lowdim_mvr}.
 
\section{Conclusion}
\label{sec:conclusion}
We have established posterior variance bounds in terms of fill distance, a geometric measure of how well the query locations cover the search space. We use this framework to study a range of exploration strategies and establish their inference regret guarantees. Among these strategies, random exploration stands out as a remarkably simple yet effective choice: it requires no additional optimization while retaining nearly the same theoretical performance as more sophisticated space-filling strategies. Our experiments further confirm its robust and efficient performance across benchmark functions and acquisition strategies. We hope this work encourages broader adoption of exploration strategies as a lightweight safeguard in practical BO systems.

\bibliographystyle{unsrtnat}
\bibliography{references}


\clearpage
\appendix


\section{Background and Notations}
Let
\[
\mathcal X
:=
\prod_{j=1}^d[a_j,b_j]
\subset\mathbb R^d,
\qquad
\mathcal X^\circ
:=
\operatorname{int}(\mathcal X)
=
\prod_{j=1}^d(a_j,b_j).
\]
Thus, $\mathcal X=\overline{\mathcal X^\circ}$ is the compact
hyperrectangular optimization domain, while Sobolev spaces are defined
on the open hyperrectangle $\mathcal X^\circ$. The domain
$\mathcal X^\circ$ is a bounded Lipschitz domain and satisfies a
uniform interior cone condition.

Given query locations
$
X_n
:=
\{x_1,\ldots, x_n\}
\subset\mathcal X,
$
define their fill distance in $\mathcal X$ by
\begin{equation}
\label{EQ:FILL_DISTANCE_DEF}
h_{X_n}
:=
\sup_{x\in\mathcal X}
\min_{1\le i\le n}
\| x-x_i\|_2.
\end{equation}
Since
$
 x
\mapsto
\min_{1\le i\le n}
\| x-x_i\|_2
$
is continuous and $\mathcal X^\circ$ is dense in $\mathcal X$,
\begin{equation}
\label{EQ:FILL_DISTANCE_INTERIOR}
h_{X_n}
=
\sup_{ x\in\mathcal X^\circ}
\min_{1\le i\le n}
\| x- x_i\|_2.
\end{equation}

Let
$
k:\mathcal X\times\mathcal X\to\mathbb R
$
be a positive-definite kernel with reproducing kernel
Hilbert space $\mathcal H_k(\mathcal X)$. For
$ x\in\mathcal X$, define the kernel vector
\[
k_n( x)
:=
\begin{bmatrix}
k( x_1, x)\\
\vdots\\
k(x_n,x)
\end{bmatrix}
\in\mathbb R^n
\]
and the kernel matrix
\[
 K_{nn}
:=
\bigl[
k(x_i,x_j)
\bigr]_{i,j=1}^n
\in\mathbb R^{n\times n}.
\]
Let
$
 I\in\mathbb R^{n\times n}
$
denote the identity matrix. With a nugget
$\lambda>0$, the GP posterior variance is
\begin{equation}
\label{EQ:POST_VAR_FILL}
\sigma_n^2(x)
=
k(x,x)
-
 k_n(x)^\top
\bigl(
 K_{nn}
+
\lambda I
\bigr)^{-1}
 k_n(x).
\end{equation}
Finally, throughout this section we write $W_2^\tau(\mathcal X^\circ)$ for the Sobolev space of order $\tau$ on the interior $\mathcal X^\circ$ of $\mathcal X$, that is, the space of square-integrable functions on $\mathcal X^\circ$ whose (weak) derivatives up to order $\tau$ are also square-integrable. We now present the lemmas used to bound the posterior standard deviation.

\begin{lemma}[Sampling inequality, Theorem~3.2 of
\citet{arcangeli2012extension}]
\label{LEMMA:SAMPLING_FILL}
Let $\tau>d/2$. For
$g\in W_2^\tau(\mathcal X^\circ)$, let
\[
 g_n
:=
\begin{bmatrix}
g(x_1)\\
\vdots\\
g(x_n)
\end{bmatrix}
\in\mathbb R^n,
\]
where point evaluations are understood using the continuous
representative of $g$ on
$\overline{\mathcal X^\circ}=\mathcal X$.
Then there exist constants $C>0$ and $h_0>0$ such that, whenever $h_{X_n}\le h_0$,
\begin{equation}
\label{EQ:SAMPLING_FILL}
\|g\|_{L_2(\mathcal X^\circ)}
\le
C
\left(
h_{X_n}^{\tau}
\|g\|_{W_2^\tau(\mathcal X^\circ)}
+
h_{X_n}^{d/2}
\|g_n\|_2
\right).
\end{equation}
\end{lemma}

\begin{lemma}[Gagliardo--Nirenberg--Sobolev inequality, Corollary~1 of \citet{brezis2019sobolev}]
\label{LEMMA:GNS_FILL}
Let $\tau>d/2$ and set
$
a = 1-\frac{d}{2\tau} \in(0,1).
$
Then every $g\in W_2^\tau(\mathcal X^\circ)$ satisfies
\begin{equation}
\label{EQ:GNS_FILL}
\|g\|_{L_\infty(\mathcal X^\circ)}
\le
C
\|g\|_{L_2(\mathcal X^\circ)}^a
\|g\|_{W_2^\tau(\mathcal X^\circ)}^{1-a}.
\end{equation}
\end{lemma}

Now we state the posterior variance bound in terms of fill-distance and provide the proof of Theorem \ref{THM:PV_FILL}.
\section{Proof of Theorem \ref{THM:PV_FILL}}
\begin{proof}
Fix $x\in\mathcal X$ and define
$
k_{x}
:=
k(x,\cdot)
\in
\mathcal H_k(\mathcal X).
$
For a coefficient vector
$
\boldsymbol c
:=
(c_1,\ldots,c_n)^\top
\in\mathbb R^n,
$
define
\begin{equation}
\label{EQ:J_DEF_FILL}
J_{x}(\boldsymbol c)
:=
\left\|
k_{x}
-
\sum_{i=1}^n
c_i k(x_i,\cdot)
\right\|_{\mathcal H_k(\mathcal X)}^2
+
\lambda\|\boldsymbol c\|_2^2.
\end{equation}
By the reproducing property,
$
\left\langle
k(u,\cdot),
k(v,\cdot)
\right\rangle_{\mathcal H_k(\mathcal X)}
=
k(u,v),
~
u, v\in\mathcal X.
$
Therefore,
\begin{equation}
\label{EQ:J_FILL}
J_{x}(\boldsymbol c)
=
k(x,x)
-
2\boldsymbol c^\top
 k_n(x)
+
\boldsymbol c^\top
\bigl(
 K_{nn}
+
\lambda I
\bigr)
\boldsymbol c.
\end{equation}
Since
$
K_{nn}
+
\lambda I
$
is positive definite, the unique minimizer of $J_{x}$ is
\begin{equation}
\label{EQ:C_HAT_FILL}
\widehat{\boldsymbol c}_{x}
:=
\bigl(
K_{nn}
+
\lambda I
\bigr)^{-1}
 k_n(x).
\end{equation}
Substituting \eqref{EQ:C_HAT_FILL} into
\eqref{EQ:J_FILL} gives
\begin{equation}
\label{EQ:J_VAR_FILL}
J_{x}
\bigl(
\widehat{\boldsymbol c}_{x}
\bigr)
=
\sigma_n^2(x).
\end{equation}

Define a scalar-valued residual function
\begin{equation}
\label{EQ:R_FILL}
r_{x}(\cdot)
:=
k(x,\cdot)
-
\sum_{i=1}^n
\bigl(
\widehat{\boldsymbol c}_{x}
\bigr)_i
k(x_i,\cdot).
\end{equation}
From \eqref{EQ:J_DEF_FILL},
\eqref{EQ:J_VAR_FILL}, and \eqref{EQ:R_FILL}
\begin{equation}
\label{EQ:R_ENERGY_FILL}
\|r_{x}\|_{\mathcal H_k(\mathcal X)}^2
+
\lambda
\|\widehat{\boldsymbol c}_{x}\|_2^2
=
\sigma_n^2(x).
\end{equation}
Hence
\begin{equation}
\label{EQ:R_BASIC_FILL}
\|r_{x}\|_{\mathcal H_k(\mathcal X)}
\le
\sigma_n(x),
\qquad
\sqrt{\lambda}
\|\widehat{\boldsymbol c}_{x}\|_2
\le
\sigma_n(x).
\end{equation}

Define the residual-evaluation vector
\[
\boldsymbol r_{x,n}
:=
\begin{bmatrix}
r_{x}(x_1)\\
\vdots\\
r_{x}(x_n)
\end{bmatrix}
\in\mathbb R^n.
\]
From \eqref{EQ:R_FILL},
$
\boldsymbol r_{x,n}
=
k_n(x)
-
K_{nn}
\widehat{\boldsymbol c}_{x},
$
and by \eqref{EQ:C_HAT_FILL},
$
\bigl(
K_{nn}
+
\lambda I
\bigr)
\widehat{\boldsymbol c}_{x}
=
k_n(x),
$
which gives us
$
\boldsymbol r_{x,n}
=
\lambda
\widehat{\boldsymbol c}_{x}.
$
Using \eqref{EQ:R_BASIC_FILL},
\begin{equation}
\label{EQ:R_DESIGN_BOUND_FILL}
\|\boldsymbol r_{x,n}\|_2
\le
\sqrt{\lambda}\,
\sigma_n(x).
\end{equation}

Evaluating \eqref{EQ:R_FILL} at the point
$x$ gives
\begin{align}
r_{x}(x)
&=
k(x,x)
-
\widehat{\boldsymbol c}_{x}^{\top}
k_n(x)
\notag\\
&=
k(x,x)
-
k_n(x)^\top
\bigl(
K_{nn}
+
\lambda I
\bigr)^{-1}
k_n(x)
\notag\\
&=
\sigma_n^2(x).
\label{EQ:R_AT_X_FILL}
\end{align}

We now use Assumption
\ref{ASSUMP:KERNEL} and \eqref{EQ:R_BASIC_FILL} to get,
\begin{equation}
\label{EQ:R_SOBOLEV_FILL}
\|r_{x}|_{\mathcal X^\circ}\|
_{W_2^\tau(\mathcal X^\circ)}
\lesssim
\|r_{x}\|_{\mathcal H_k(\mathcal X)}
\lesssim
\sigma_n(x).
\end{equation}

Applying Lemma~\ref{LEMMA:SAMPLING_FILL} to \(r_x|_{\mathcal X^\circ}\), together with \eqref{EQ:R_DESIGN_BOUND_FILL} and \eqref{EQ:R_SOBOLEV_FILL}, yields, whenever \(h_{X_n}\le h_0\),
\begin{align}
\|r_{x}\|_{L_2(\mathcal X^\circ)}
&\lesssim 
h_{X_n}^{\tau}
\|r_{x}\|_{W_2^\tau(\mathcal X^\circ)}
+
h_{X_n}^{d/2}
\|\boldsymbol r_{x,n}\|_2
\notag\\
&\lesssim
\sigma_n(x)
\left(
h_{X_n}^{\tau}
+
\sqrt{\lambda}\,
h_{X_n}^{d/2}
\right).
\label{EQ:R_L2_FILL}
\end{align}
Let
$
a
:=
1-\frac{d}{2\tau}.
$
By Lemma~\ref{LEMMA:GNS_FILL} with \eqref{EQ:R_SOBOLEV_FILL} and \eqref{EQ:R_L2_FILL},
\begin{align}
\|r_{x}\|_{L_\infty(\mathcal X^\circ)}
&\lesssim
\|r_{x}\|_{L_2(\mathcal X^\circ)}^a
\|r_{x}\|_{W_2^\tau(\mathcal X^\circ)}^{1-a}
\notag\\
&\lesssim
\sigma_n(x)
\left(
h_{X_n}^{\tau}
+
\sqrt{\lambda}\,
h_{X_n}^{d/2}
\right)^a.
\label{EQ:R_LINF_FILL}
\end{align}

Since $r_{x}$ is continuous on
$\mathcal X=\overline{\mathcal X^\circ}$ and
$\mathcal X^\circ$ is dense in $\mathcal X$,
$
|r_{x}(x)|
\le
\|r_{x}\|_{L_\infty(\mathcal X^\circ)}.
$
From \eqref{EQ:R_AT_X_FILL} and
\eqref{EQ:R_LINF_FILL}, we get
\begin{equation}
\label{EQ:SIGMA_PRE_FINAL}
\sigma_n^2(x)
\lesssim
\sigma_n(x)
\left(
h_{X_n}^{\tau}
+
\sqrt{\lambda}\,
h_{X_n}^{d/2}
\right)^a.
\end{equation}
If $\sigma_n(x)>0$, dividing
\eqref{EQ:SIGMA_PRE_FINAL} by
$\sigma_n(x)$ and squaring gives
\[
\sigma_n^2(x)
\lesssim
\left(
h_{X_n}^{\tau}
+
\sqrt{\lambda}\,
h_{X_n}^{d/2}
\right)^{2a}.
\]
The same inequality is immediate when
$\sigma_n(x)=0$. Since $(u+v)^p \lesssim u^p + v^p$ for all $u, v \ge 0$ and any fixed $p > 0$, we obtain
\[
\sigma_n^2(\boldsymbol{x}) \lesssim h_{\mathcal{X}_n}^{2a\tau} + \lambda^a h_{\mathcal{X}_n}^{da}.
\]
Finally, since
$
2a\tau=2\tau-d,~
a=1-\frac{d}{2\tau},
$
and hence
\[
\sigma_n^2(x)
\lesssim
h_{X_n}^{2\tau-d}
+
\lambda^{1-\frac{d}{2\tau}}
h_{X_n}^{d\left(1-\frac{d}{2\tau}\right)}.
\]
Taking the supremum over
$x\in\mathcal X$, we get the final statement.
\end{proof}

\noindent The following two lemmas verify that both the Matérn and SE kernels satisfy Assumption \ref{ASSUMP:KERNEL}.

\begin{lemma}[Mat\'ern kernel]
\label{LEM:MATERN_FILL_VARIANCE}
Let
\[
k_\nu:
\mathcal X\times\mathcal X
\to\mathbb R
\]
be a Mat\'ern kernel with
$\nu>0$. Then, for all sufficiently small
$h_{X_n}$,
\begin{equation}
\label{EQ:MATERN_FILL_VARIANCE}
\sup_{x\in\mathcal X}
\sigma_n^2(x)
\le
C
\left\{
h_{X_n}^{2\nu}
+
\lambda^{\frac{2\nu}{2\nu+d}}
h_{X_n}^{\frac{2\nu d}{2\nu+d}}
\right\}.
\end{equation}
\end{lemma}

\begin{proof}
Define the restriction of the Mat\'ern kernel to the open
hyperrectangle by
$
k_{\nu,\mathcal X^\circ}
:= k_\nu|_{\mathcal X^\circ\times\mathcal X^\circ}.
$
By Proposition~3.3 of
\citet{teckentrup2020convergence},
$
\mathcal{H}_{k_{\nu,\mathcal{X}^\circ}}(\mathcal{X}^\circ) \cong W_2^{\nu+d/2}(\mathcal{X}^\circ).
$ 
Hence
\begin{equation}
\label{EQ:MATERN_OPEN_EMBED}
\|u\|_{W_2^{\nu+d/2}(\mathcal X^\circ)}
\lesssim
\|u\|_{
\mathcal H_{k_{\nu,\mathcal X^\circ}}(\mathcal X^\circ)
}.
\end{equation}
Now let
$
g\in\mathcal H_{k_\nu}(\mathcal X).
$
By Theorem~5.7 of
\citet{paulsen2016rkhs}, applied to the inclusion
$\mathcal X^\circ\hookrightarrow\mathcal X$,
\[
\|g|_{\mathcal X^\circ}\|_{
\mathcal H_{k_{\nu,\mathcal X^\circ}}(\mathcal X^\circ)
}
\le
\|g\|_{\mathcal H_{k_\nu}(\mathcal X)}.
\]
Combining this with \eqref{EQ:MATERN_OPEN_EMBED},
\[
\|g|_{\mathcal X^\circ}\|_
{W_2^{\nu+d/2}(\mathcal X^\circ)}
\lesssim
\|g\|_{\mathcal H_{k_\nu}(\mathcal X)}.
\]
Thus Assumption
\ref{ASSUMP:KERNEL} holds with
$
\tau
=
\nu+\frac d2.
$
Theorem~\ref{THM:PV_FILL} therefore gives
\eqref{EQ:MATERN_FILL_VARIANCE}.
\end{proof}

\begin{lemma}[Squared-exponential kernel]
\label{LEM:SE_FILL_VARIANCE}
Let $k=k_{\mathrm{SE}}$ be a squared-exponential kernel with fixed
hyperparameters. Then, for every integer $m>d/2$, there exist constants
$C_m>0$ and $h_m^\star>0$ such that, whenever
$h_{X_n}\le h_m^\star$,
\begin{equation}
\label{EQ:SE_FILL_VARIANCE}
\sup_{x\in\mathcal X}
\sigma_n^2(x)
\le
C_m
\left\{
h_{X_n}^{2m-d}
+
\lambda^{1-\frac{d}{2m}}
h_{X_n}^{d\left(1-\frac{d}{2m}\right)}
\right\}.
\end{equation}
\end{lemma}

\begin{proof}
Define the restriction of the squared-exponential kernel to the open
hyperrectangle by
$
k_{\mathrm{SE},\mathcal X^\circ}
:=
k_{\mathrm{SE}}|_{\mathcal X^\circ\times\mathcal X^\circ}.
$
By Theorem~4.48 of
\citet{steinwart2008svm}, for every integer $m\ge1$,
\begin{equation}
\label{EQ:SE_OPEN_EMBED}
\|u\|_{W_2^m(\mathcal X^\circ)}
\lesssim
\|u\|_{
\mathcal H_{k_{\mathrm{SE},\mathcal X^\circ}}
(\mathcal X^\circ)
},
\end{equation}
For
$
g\in\mathcal H_{k_{\mathrm{SE}}}(\mathcal X),
$
Theorem~5.7 of
\citet{paulsen2016rkhs}, applied to the inclusion
$\mathcal X^\circ\hookrightarrow\mathcal X$, gives
\[
\|g|_{\mathcal X^\circ}\|_{
\mathcal H_{k_{\mathrm{SE},\mathcal X^\circ}}
(\mathcal X^\circ)
}
\le
\|g\|_{\mathcal H_{k_{\mathrm{SE}}}(\mathcal X)}.
\]
Combining this with \eqref{EQ:SE_OPEN_EMBED},
\[
\|g|_{\mathcal X^\circ}\|_{W_2^m(\mathcal X^\circ)}
\lesssim
\|g\|_{\mathcal H_{k_{\mathrm{SE}}}(\mathcal X)}.
\]
Thus Assumption
\ref{ASSUMP:KERNEL} holds with
$
\tau=m.
$
Theorem~\ref{THM:PV_FILL} therefore gives
\eqref{EQ:SE_FILL_VARIANCE}.
\end{proof}

\section{Proof of Proposition \ref{PROP:FILL_GREEDY}}
\begin{proof}
Consider a realization with $M_T\ge2$. Let
$t_1<\cdots<t_{M_T}\le T$ denote all exploration rounds among the
first $T$ rounds, and write $z_j:=x_{t_j}$ for their query locations.
For every $i<j$, the greedy packing rule gives
\[
\|z_j-z_i\|_2
\ge
\min_{1\le s<t_j}\|z_j-x_s\|_2
=
h_{X_{t_j-1}}
\ge h_{X_T},
\]
where the final inequality follows from $X_{t_j-1}\subseteq X_T$. Thus the separation radius of
$E_T:=\{z_1,\ldots,z_{M_T}\}$ satisfies
\[
q(E_T):=\frac12\min_{i\ne j}\|z_i-z_j\|_2
\ge \frac{h_{X_T}}{2}.
\]
By \citet[Lemma~2.1]{pronzato2023quasi} and compactness of
$\mathcal{X}$, there is a constant $C>0$, depending only on
$\mathcal{X}$ and $d$, such that
\begin{equation}\label{ineq:fill_dist_prelim}
h_{X_T}\le2q(E_T)\le C M_T^{-1/d}
\qquad
\text{whenever }M_T\ge2.
\end{equation}
It remains to relate $M_T$ to its deterministic mean $e_T$.
Since $M_T$ is a sum of independent Bernoulli variables,
\[
\operatorname{Var}(M_T)
=
\sum_{t=1}^T p_t(1-p_t)
\le e_T.
\]
Fix any $\eta\in(0,1)$. The Chebyshev inequality implies
\[
\mathbb{P}\bigl(M_T\ge(1-\eta)e_T\bigr)
\ge1-\frac{1}{\eta^2e_T}.
\]
Since $e_T\to\infty$, we have $(1-\eta)e_T\ge2$ for all sufficiently
large $T$. On the event $\{M_T\ge(1-\eta)e_T\}$, the 
bound \eqref{ineq:fill_dist_prelim} therefore gives
\[
h_{X_T}
\le C M_T^{-1/d}
\le C(1-\eta)^{-1/d}e_T^{-1/d}.
\]
Consequently,
\[
\mathbb{P}\left(
h_{X_T}\lesssim C(1-\eta)^{-1/d}e_T^{-1/d}
\right)
\ge1-\frac{1}{\eta^2e_T}.
\]
For every \(\delta\in(0,1)\), after enlarging the implicit constant, if necessary, to cover all finite values of \(T\) and absorbing the fixed factor \((1-\eta)^{-1/d}\), we obtain
\[
\mathbb{P}\left(
h_{X_T}\lesssim e_T^{-1/d}
\right)\ge1-\delta.
\]
\end{proof}

\section{Proof of Proposition \ref{PROP:FILL_RANDOM}}

\begin{proof}
Let $\{Z_j\}_{j\ge1}$ be the independent exploration draws,
and define
\[
H_m:=h_{\{Z_1,\ldots,Z_m\}}.
\]
The fill-distance estimate from the proof of
\citet[Lemma~2]{oates2019convergence} gives, for all sufficiently
large $m$,
\[
\mathbb{E}[H_m]
\lesssim
\left(\frac{\log m}{m}\right)^{1/d}.
\]

Recall that $M_T$ is the number of exploration rounds by time $T$.
On $\{M_T=m\}$ with $m\ge1$, the first $m$ exploration draws belong to $X_T$. Since adding points cannot increase the fill distance, $h_{X_T}\le H_m$ on this event. Moreover, $M_T$ is independent of the exploration queries. Consequently,
\[
\mathbb{E}[h_{X_T}\mid M_T=m]
\le \mathbb{E}[H_m\mid M_T=m]
=\mathbb{E}[H_m].
\]

It remains to relate $M_T$ to its deterministic mean $e_T$.
Since $M_T$ is a sum of independent Bernoulli variables,
\[
\operatorname{Var}(M_T)
=\sum_{t=1}^T p_t(1-p_t)
\le e_T.
\]
Fix any $\eta\in(0,1)$. Chebyshev's inequality gives
\[
\mathbb{P}\bigl(M_T<(1-\eta)e_T\bigr)
\le\frac{1}{\eta^2e_T}.
\]
For sufficiently large $T$, every integer
$m\ge(1-\eta)e_T$ is large enough for the expectation bound above to apply. Since $u\mapsto(\log u)/u$ is decreasing for
$u\ge e$, for such counts,
\[
\mathbb{E}[h_{X_T}\mid M_T=m]
\lesssim
\left(\frac{\log e_T}{e_T}\right)^{1/d}.
\]
The above holds for all $T$ by enlarging the implicit constant. For the remaining counts, use
$h_{X_T}\le D:=\operatorname{diam}(\mathcal{X})$.
Averaging over $M_T$ therefore yields
\[
\mathbb{E}[h_{X_T}]
\lesssim
\left(\frac{\log e_T}{e_T}\right)^{1/d}
+\frac{D}{\eta^2e_T}
\lesssim
\left(\frac{\log e_T}{e_T}\right)^{1/d}.
\]

Finally, Markov's inequality implies that, with probability
at least $1-\delta$,
\[
h_{X_T}
\le \frac{\mathbb{E}[h_{X_T}]}{\delta}
\lesssim
\left(\frac{\log e_T}{e_T}\right)^{1/d}
=\tilde{\mathcal{O}}\bigl(e_T^{-1/d}\bigr).
\]
\end{proof}

\section{Proof of Corollary \ref{COR:PV_RATES}}
\begin{proof}
The noise-dependent term in
Theorem~\ref{THM:PV_FILL} dominates for sufficiently small
fill distance, since
$
2\tau-d>d\left(1-\frac{d}{2\tau}\right).
$
Consequently,
\[
\sup_{x\in\mathcal{X}}\sigma_T^2(x)
\lesssim h_{X_T}^{d(1-\frac{d}{2\tau})}.
\]
The fill-distance bounds ensure that the condition $h_{X_T}\le h_0$ holds on the respective events of probability at least $1-\delta$ for sufficiently large $T$. For the Mat\'ern kernel, take $\tau=\nu+d/2$. Then
\[
\sup_{x\in\mathcal{X}}\sigma_T^2(x)
\lesssim h_{X_T}^{\frac{2\nu d}{2\nu+d}}.
\]
Applying Propositions~\ref{PROP:FILL_GREEDY}
and~\ref{PROP:FILL_RANDOM} gives the stated Mat\'ern rates.

For the SE kernel, Assumption~\ref{ASSUMP:KERNEL} holds
for every integer $\tau>d/2$. Given $\varepsilon\in(0,1)$,
choose $\tau$ sufficiently large that $d/(2\tau)<\varepsilon$.
For either exploration strategy, the fill-distance bound gives
\[
\sup_{x\in\mathcal{X}}\sigma_T^2(x)
=
\tilde{\mathcal{O}}\left(
e_T^{-(1-\frac{d}{2\tau})}
\right).
\]
Since $1-d/(2\tau)>1-\varepsilon$, the logarithmic factors
are absorbed by this strict difference in exponents.
Thus the bound is
$\mathcal{O}(e_T^{-(1-\varepsilon)})$. Finally, the implicit constants may be enlarged, if necessary, to cover the finitely many values of $T$ below the threshold
for which the preceding asymptotic bounds apply.
\end{proof}

\section{Proof for inference regret bound}
\begin{proposition}[Sequential confidence bound]
\label{PROP:FLYNN_REEB}
Suppose that $f\in\mathcal H_k(\mathcal X)$ with $\|f\|_{\mathcal H_k}\le B$, and that the observation noise is i.i.d. and $R$-sub-Gaussian. Let $\lambda>0$ and define
\[
\gamma_t(\lambda)
:=
\max_{x_1,\ldots,x_t\in\mathcal X}
\frac{1}{2}
\log\det\left(I+\lambda^{-1}K_t\right).
\]
Then, for every $\delta\in(0,1)$, with probability at least $1-\delta$,
simultaneously for all $t\ge1$ and all $x\in\mathcal X$,
\begin{equation}
\label{EQN:FLYNN_REEB}
|f(x)-\mu_t(x)|
\le
\left[
B^2+
\frac{2R^2}{\lambda}
\left(
\gamma_t(\lambda)+\log\frac{1}{\delta}
\right)
\right]^{1/2}
\sigma_t(x).
\end{equation}
\end{proposition}

\begin{proof}
Apply Corollary~4.3 of \citet{flynn2025tighter} with their
regularization parameter $\alpha=\lambda$ and covariance-scale parameter
$c=R^2/\lambda$. Their confidence sequence holds simultaneously for all
$t\ge1$ and remains valid when each query $x_t$ depends on the previous
observations. By Lemma~B.3 and C.2 therein, the resulting confidence
set implies
\[
|f(x)-\mu_t(x)|
\le
\frac{\tilde R_{\lambda,t}}{\sqrt{\lambda}}\sigma_t(x), \qquad \tilde R_{\lambda,t}^{\,2}
=
R^2\log\det\left(I+\lambda^{-1}K_t\right)
+\lambda B^2
+2R^2\log\frac{1}{\delta}.
\]
simultaneously for all $t\ge1$ and $x\in\mathcal X$. Since
$
\log\det\left(I+\lambda^{-1}K_t\right)
\le 2\gamma_t(\lambda),
$
the result follows.
\end{proof}
\subsection{Proof of Theorem \ref{THM:INFERENCE_GENERAL}}
\begin{proof}
Since $\widehat{x}_T:=\arg\max_{x\in\mathcal X}\mu_T(x)$, we have
$\mu_T(x^\star)\le\mu_T(\widehat{x}_T)$, and hence
\begin{align}
I_T
&=
f(x^\star)-f(\widehat{x}_T) \nonumber\\
&\le
|f(x^\star)-\mu_T(x^\star)|
+
|\mu_T(\widehat{x}_T)-f(\widehat{x}_T)| \nonumber\\
&\le
2\sup_{x\in\mathcal X}|f(x)-\mu_T(x)|.
\label{EQN:INF_BASIC}
\end{align}
By Proposition~\ref{PROP:FLYNN_REEB}, with probability at least
$1-\delta/2$,
\begin{equation}
\label{EQN:INF_CONF}
\sup_{x\in\mathcal X}|f(x)-\mu_T(x)|
\le
\left[
B^2+
\frac{2R^2}{\lambda}
\left(
\gamma_T(\lambda)+\log\frac{2}{\delta}
\right)
\right]^{1/2}
\sup_{x\in\mathcal X}\sigma_T(x).
\end{equation}
Moreover, Corollary~\ref{COR:PV_RATES} implies that, with probability at least $1-\delta/2$,
\[
\sup_{x\in\mathcal X}\sigma_T(x)
=
\begin{cases}
\displaystyle
\tilde{\mathcal O}\left(
e_T^{-\frac{\nu}{2\nu+d}}
\right),
& \text{Mat\'ern},\\
\displaystyle
\mathcal O\left(
e_T^{-\frac{1-\varepsilon}{2}}
\right),
& \text{SE},
\end{cases}
\]
for any $\varepsilon\in(0,1)$. By a union bound, the confidence bound
\eqref{EQN:INF_CONF} and the above posterior variance bound therefore hold
simultaneously with probability at least $1-\delta$. For a given nugget $\lambda>0$, the maximum information gain bounds \citep{vakili2021information} satisfy
\[
\gamma_T(\lambda)
=
\begin{cases}
\displaystyle
\tilde{\mathcal O}\left(
T^{\frac{d}{2\nu+d}}
\right),
& \text{Mat\'ern},\\
\displaystyle
\mathcal O\left((\log T)^{d+1}\right),
& \text{SE}.
\end{cases}
\]
Substituting these rates and the posterior variance bounds into \eqref{EQN:INF_CONF} yields, with probability at least $1-\delta$,
\[
I_T
=
\begin{cases}
\displaystyle
\tilde{\mathcal O}\left(
T^{\frac{d}{2(2\nu+d)}}
e_T^{-\frac{\nu}{2\nu+d}}
\right),
& \text{Mat\'ern},\\
\displaystyle
\tilde{\mathcal O}\left(
e_T^{-\frac{1-\varepsilon}{2}}
\right),
& \text{SE}.
\end{cases}
\]
\end{proof}

\section{Additional Experiments}
\label{APPEND:ADD_EXP}
\makeatletter
\@ifundefined{FloatBarrier}{%
  \def\app@fb@botlist{\@botlist}%
  \def\app@fb@topbarrier{\suppressfloats[t]}%
  \def\FloatBarrier{\par\begingroup\let\@elt\relax
    \edef\@tempa{\app@fb@botlist\@deferlist\@dbldeferlist}%
    \ifx\@tempa\@empty
    \else
      \ifx\@fltovf\relax
        \if@firstcolumn\clearpage
        \else\null\newpage\FloatBarrier\fi
      \else
        \newpage\let\@fltovf\relax\FloatBarrier
      \fi
    \fi
    \endgroup\app@fb@topbarrier}%
}{}
\makeatother
This appendix provides complete experimental details and additional results for Section~\ref{sec:experiments}. We first describe the experimental settings used throughout and present additional results for the benchmark test function examples. We then report further results for the two YAHPO tasks and the Lunar Lander controller-tuning problem. We conclude with a comparison of the computational efficiency of BO with uniform random exploration, MVR-aided BO, and standard BO on two-dimensional benchmarks.

\subsection{Experimental details}
\label{app:exp_protocol}

\paragraph{Experimental design.} Table~\ref{tab:app_protocol} summarizes the evaluation budgets, initial designs, and numbers of independent runs. For each synthetic benchmark function, an experimental run begins with two scrambled-Sobol observations, followed by 400 BO evaluations; accordingly, the horizontal axes in Figures~\ref{fig:app_synthetic_tuned} and~\ref{fig:app_synthetic_uniform_alpha_half} correspond to these 400 BO evaluations. For each objective--acquisition pair in the six-objective benchmark study, we consider the random-exploration schedule
$
p_t=\min\{1,\log(t+1)/(t+1)^\alpha\},
$
with \(\alpha\in\{1/4,1/3,1/2,2/3,3/4\}\), where exploration points are sampled uniformly over the search space. We conduct 20 independent runs per setting.

For the YAHPO and Lunar Lander studies, the reported budgets include the initial designs. In both applications, we fix \(\alpha=1/2\) and likewise sample exploration points uniformly over the corresponding search space. In addition to uniform random sampling, both applications also evaluate approximate greedy packing as an alternative exploration rule. We conduct 30 independent runs for each application setting.

Finally, to assess the computational cost of different exploration strategies, we compare standard BO with two exploration-aided variants: uniform-exploration BO, in which exploration points are drawn uniformly at random over the search space, and MVR-exploration BO, in which exploration points are selected using the maximum variance reduction criterion. We use GP-UCB as the acquisition function, set \(\alpha=1/2\) for both exploration strategies, and conduct 15 independent runs per setting on a shifted two-dimensional Ackley function and a shifted and rotated two-dimensional Rastrigin function. 
We report both final inference regret and runtime; further details are provided in Section~\ref{app:exp_lowdim_mvr}.

\begin{table}[!htb]
\setlength{\belowcaptionskip}{5pt}
\centering\small
\caption{Evaluation budgets, initial designs, and numbers of runs.} 
\label{tab:app_protocol}
\begin{tabular}{@{}lrrrr@{}}
\toprule
Study & Dimension & Budget & Initial & Number of Runs\\
\midrule
Benchmark function comparison & 10, 20, 30 & 400 BO & 2 & 20\\
Uniform--MVR comparison & 2 & 400 BO & 2 & 15\\
YAHPO car / sylvine & 14 / 28 & 120 / 200 total & 10 & 30\\
Lunar Lander & 12 & 400 total & 24 & 30\\
\bottomrule
\end{tabular}
\end{table}

\paragraph{Exploration rules.}
Standard BO (denoted Standard in figures and tables) follows the base acquisition rule at every round. Uniform random exploration (denoted Uniform) instead samples a point uniformly from the search domain on exploration rounds and otherwise follows the base acquisition rule. The hyperparameter optimization experiments additionally consider approximate greedy packing, which at the $t^{\text{th}}$ round selects the candidate farthest from the previously evaluated inputs over a nested scrambled Sobol set of size $256+16t$. For each task and acquisition function, all methods are evaluated using the same set of random seeds. Runs sharing a seed use the same initial design and, when observation noise is present, the same noise realizations. When exploration rules are compared under the same schedule, a shared Bernoulli draw with probability $p_t$ determines whether exploration occurs at each round. On non-exploration rounds, each method follows the base acquisition rule. 

\paragraph{Model and acquisition functions.} 
The synthetic and YAHPO experiments use Mat\'ern-($3/2$) GPs with unit signal variance. Before fitting the GP, we center the observed objective values by subtracting their sample mean and add the mean back when computing posterior predictions. Lengthscales are refitted every five evaluations and constrained to $[0.01,20]$. We use multi-start optimization with 20 starting points for both acquisition function optimization and posterior mean optimization. For each recommendation, we compare the posterior mean optimizer with all previously evaluated inputs and select the point with the largest posterior mean.
For the synthetic benchmark experiments, both the observation-noise variance and the GP nugget are $0.01$. YAHPO uses no additional observation noise, but we include a small nugget of $10^{-4}$. GP-UCB uses an exploration coefficient of two, corresponding to a factor of two multiplying the posterior standard deviation in the UCB acquisition function. At $t^{\text{th}}$ round, GP-TS draws a fresh Sobol candidate set of size \(\operatorname{clip}(\lceil64\sqrt{t}\rceil,128,512)\). It then samples from the GP posterior over this set, with the covariance inflated by a factor of \(\log(t+1)\), and selects the candidate with the largest sampled value.

\paragraph{Computing infrastructure.} All synthetic, YAHPO, and Lunar Lander experiments were run on an internal Linux x86\_64 server with 128 logical CPUs and without GPU acceleration. Independent runs were executed in parallel, while each BO run used a single process. The runtime comparison between uniform random exploration and MVR-based BO was conducted on the same server under a common experimental configuration. 

\paragraph{Assets and licenses.} We use YAHPO Gym v1.0.2 and its accompanying benchmark data, both distributed under the Apache License 2.0. The Lunar Lander environment is provided by Gymnasium under the MIT License and uses the Box2D physics engine, distributed under the zlib License. The 12-parameter controller follows the example distributed with TuRBO, whose reference implementation is released under the Uber Non-Commercial License.

\paragraph{Performance measure and uncertainty.}
Inference regret in the benchmark function experiments is evaluated at the maximizer of the current posterior mean. It therefore measures the quality of the recommendation at that evaluation step, rather than the best value observed so far, and its trajectory need not be monotone. Trajectory figures report either one standard deviation or 1.96 times the standard error, as stated in their captions. When comparing two methods, we compute paired differences within each random seed and bootstrap these differences to construct pointwise 95$\%$ confidence intervals.

\subsection{Benchmark function experiments}
\label{app:exp_synthetic}
The benchmark function experiments consider six objectives: 10-dimensional Trid, 20-dimensional Rosenbrock, and 30-dimensional shifted Bent Cigar, shifted Sphere, shifted Griewank, and normalized Sum of Different Powers. Together, these functions span a range of optimization challenges, including curved valleys, ill-conditioned directions, separable structure, and multimodality. All inputs are represented on the unit hypercube and then mapped to their respective benchmark domains.

\paragraph{Benchmark sources.}
The base definitions of Trid, Rosenbrock, Sphere, Griewank, and Sum of Different Powers are documented in the Virtual Library of Simulation Experiments \citep{simulationlib}. Trid 10, Rosenbrock, Sphere, and Griewank are also covered by \citet{jamil2013literature}; the original Rosenbrock function was introduced by \citet{rosenbrock1960automatic}. For Sum of Different Powers, we use the base form $\sum_{i=1}^{d}|x_i|^{i+1}$ described by \citet{molga2005test}, with output normalization. The base Bent Cigar function is documented by \citet{Al-Roomi2015}.

Our use of transformed benchmark functions follows established practice in numerical optimization. For example, CEC 2005 includes shifted Sphere and shifted, rotated Griewank \citep{suganthan2005problem}, while shifted and rotated Bent Cigar is included in CEC 2017 \citep{awad2016cec2017} and evaluated in 30 dimensions by \citet{ye2024multistrategy}.

\paragraph{Calibrated exploration schedule.} 
The exponent $\alpha$ controls the frequency of uniform random exploration through the schedule $p_t=\min\{1,\log(t+1)/(t+1)^\alpha\}$, with smaller values corresponding to more frequent exploration. For each objective--acquisition pair, we compare $\alpha\in\{1/4,1/3,1/2,2/3,3/4\}$ and select the value that minimizes the final inference regret averaged over 20 random seeds; we denote this selected value by $\alpha^*$. Figure~\ref{fig:app_synthetic_tuned} shows the resulting trajectories, and Table~\ref{tab:app_synthetic_endpoints} reports the corresponding final inference regret values. BO with uniform exploration achieves lower mean final inference regret than standard BO in 17 of the 18 scenarios, with the paired 95$\%$ bootstrap confidence interval excluding zero in 16 cases. 

\paragraph{Common exploration schedule.}We next fix $\alpha=1/2$ across all 18 objective--acquisition pairs, rather than selecting the exploration exponent separately for each setting. Figure~\ref{fig:app_synthetic_uniform_alpha_half} shows the corresponding inference regret trajectories, and Table~\ref{tab:app_synthetic_alpha_half} reports the final values. Even under this common schedule, incorporating uniform random exploration remains beneficial: BO with uniform random exploration achieves lower mean final inference regret than standard BO in 16 of the 18 scenarios, with statistically significant improvements in 11 cases based on the paired 95$\%$ bootstrap confidence intervals and no statistically significant losses. These results indicate that the benefit of explicit exploration persists even without careful calibration of the exploration parameter. 

\paragraph{Exploration parameter sensitivity.}
Figure~\ref{fig:exp_synthetic_alpha_sweep} in the main body examines the sensitivity of final inference regret to the exploration exponent \(\alpha\), with smaller values corresponding to more frequent uniform random exploration. Over 400 BO rounds, the expected numbers of exploration queries are 399.4, 364.2, 164.0, 75.2, and 51.6 for \(\alpha=1/4,1/3,1/2,2/3,3/4\), respectively. Among the five values considered, \(\alpha=1/4\) achieves the lowest mean final inference regret in 16 of the 18 objective--acquisition pairs. The two exceptions are GP-TS on Shifted Bent Cigar 30D and Sum of Different Powers 30D, for which \(\alpha=1/3\) and \(\alpha=3/4\), respectively, attain the lowest means. Overall, the results favor more aggressive exploration across most benchmark settings, while also indicating that the best exploration frequency can depend on the objective and acquisition rule.

\begin{figure*}[htp]
\centering
\includegraphics[width=\textwidth]{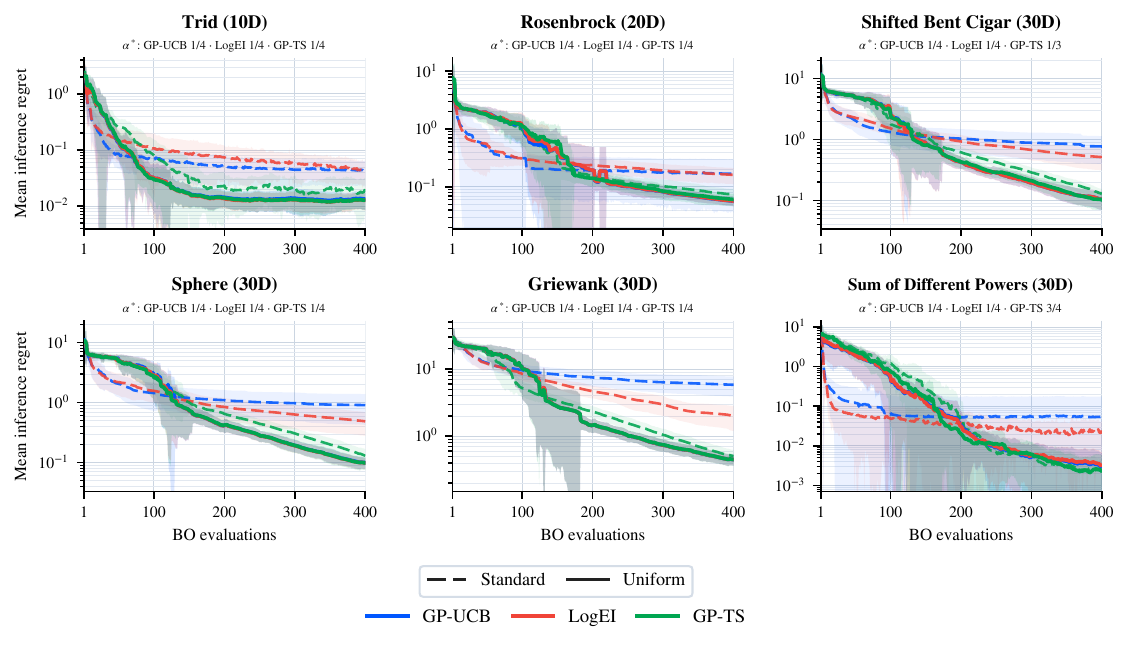}
\caption{\textbf{Benchmark function inference regret trajectories under calibrated exploration.}
Standard BO (Standard) and BO with uniform random exploration (Uniform) are compared using GP-UCB, LogEI, and GP-TS across six benchmark functions. For each objective--acquisition pair, \(\alpha^*\) is selected from the five-value grid: $\{1/4,1/3,1/2,2/3,3/4\}$ to minimize averaged final inference regret. Dashed and solid curves denote Standard and Uniform, respectively, while colors indicate the acquisition function. Curves show the mean inference regret over 20 random seeds, with shaded regions representing one standard deviation. Lower values indicate better performance. Because inference regret is evaluated at the current posterior mean maximizer, the trajectories need not be monotone.}
\label{fig:app_synthetic_tuned}
\end{figure*}

\begin{table*}[htp]
\setlength{\belowcaptionskip}{5pt}
\centering\scriptsize
\caption{\textbf{Final inference regret under calibrated exploration.} For each objective--acquisition pair, \(\alpha^*\) is selected from \(\{1/4,1/3,1/2,2/3,3/4\}\) to minimize mean final inference regret. Standard denotes standard BO, while Uniform denotes BO with uniform random exploration. Table entries are based on 20 independent runs. Gains measure the improvement of Uniform over Standard; bold values indicate gains whose pointwise 95\% bootstrap confidence intervals exclude zero. Lower inference regret is better.}
\label{tab:app_synthetic_endpoints}
\resizebox{\textwidth}{!}{%
\begin{tabular}{@{}llcllll@{}}
\toprule
Objective & Dim. & Acquisition & $\alpha^*$ & Standard (mean $\pm$ SD) & Uniform (mean $\pm$ SD) & Gain [95\% CI]\\
\midrule
Trid & 10 & GP-UCB & $1/4$ & 0.043 $\pm$ 0.021 & 0.013 $\pm$ 0.005 & \textbf{0.030 [0.021,0.039]}\\
Trid & 10 & LogEI & $1/4$ & 0.044 $\pm$ 0.024 & 0.013 $\pm$ 0.004 & \textbf{0.032 [0.022,0.043]}\\
Trid & 10 & GP-TS & $1/4$ & 0.019 $\pm$ 0.012 & 0.013 $\pm$ 0.004 & \textbf{0.006 [0.001,0.012]}\\
\addlinespace[1pt]
Rosenbrock & 20 & GP-UCB & $1/4$ & 0.168 $\pm$ 0.132 & 0.056 $\pm$ 0.010 & \textbf{0.112 [0.066,0.179]}\\
Rosenbrock & 20 & LogEI & $1/4$ & 0.162 $\pm$ 0.051 & 0.057 $\pm$ 0.011 & \textbf{0.105 [0.083,0.127]}\\
Rosenbrock & 20 & GP-TS & $1/4$ & 0.075 $\pm$ 0.017 & 0.061 $\pm$ 0.014 & \textbf{0.014 [0.003,0.025]}\\
\addlinespace[1pt]
Shifted Bent Cigar & 30 & GP-UCB & $1/4$ & 0.769 $\pm$ 0.370 & 0.104 $\pm$ 0.035 & \textbf{0.665 [0.516,0.832]}\\
Shifted Bent Cigar & 30 & LogEI & $1/4$ & 0.514 $\pm$ 0.197 & 0.105 $\pm$ 0.034 & \textbf{0.409 [0.326,0.498]}\\
Shifted Bent Cigar & 30 & GP-TS & $1/3$ & 0.127 $\pm$ 0.025 & 0.102 $\pm$ 0.033 & \textbf{0.025 [0.007,0.042]}\\
\addlinespace[1pt]
Sphere & 30 & GP-UCB & $1/4$ & 0.902 $\pm$ 0.479 & 0.099 $\pm$ 0.020 & \textbf{0.802 [0.613,1.014]}\\
Sphere & 30 & LogEI & $1/4$ & 0.485 $\pm$ 0.207 & 0.098 $\pm$ 0.020 & \textbf{0.387 [0.300,0.477]}\\
Sphere & 30 & GP-TS & $1/4$ & 0.129 $\pm$ 0.036 & 0.098 $\pm$ 0.023 & \textbf{0.031 [0.011,0.054]}\\
\addlinespace[1pt]
Griewank & 30 & GP-UCB & $1/4$ & 5.827 $\pm$ 1.922 & 0.449 $\pm$ 0.083 & \textbf{5.379 [4.579,6.196]}\\
Griewank & 30 & LogEI & $1/4$ & 2.001 $\pm$ 0.807 & 0.449 $\pm$ 0.083 & \textbf{1.552 [1.213,1.910]}\\
Griewank & 30 & GP-TS & $1/4$ & 0.501 $\pm$ 0.116 & 0.449 $\pm$ 0.088 & 0.051 [-0.006,0.109]\\
\addlinespace[1pt]
Sum of Different Powers & 30 & GP-UCB & $1/4$ & 0.054 $\pm$ 0.122 & 0.003 $\pm$ 0.004 & \textbf{0.051 [0.009,0.110]}\\
Sum of Different Powers & 30 & LogEI & $1/4$ & 0.024 $\pm$ 0.019 & 0.003 $\pm$ 0.004 & \textbf{0.021 [0.013,0.030]}\\
Sum of Different Powers & 30 & GP-TS & $3/4$ & 0.002 $\pm$ 0.002 & 0.002 $\pm$ 0.004 & $-0.0001$ [$-0.0020$,0.0016]\\
\bottomrule
\end{tabular}}
\end{table*}
\FloatBarrier

\begin{figure*}[htp]
\centering
\includegraphics[width=\textwidth]{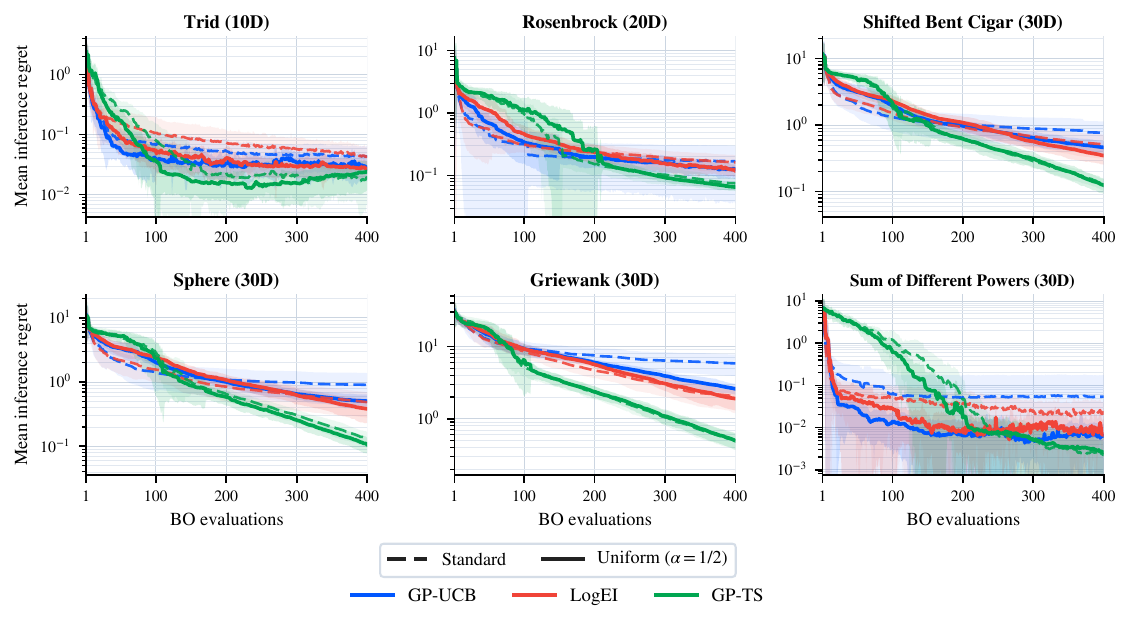}
\caption{\textbf{Benchmark function inference regret trajectories at a fixed exploration exponent.}
Standard BO (Standard) and BO with uniform random exploration (Uniform) are compared using GP-UCB, LogEI, and GP-TS across six benchmark functions, with the exploration exponent fixed at $\alpha=1/2$. Dashed and solid curves denote Standard and Uniform, respectively, while colors indicate the acquisition function. Curves show the mean inference regret over 20 random seeds, with shaded regions representing one standard deviation. Lower values indicate better performance.} 


\label{fig:app_synthetic_uniform_alpha_half}
\end{figure*}

\begin{table*}[!t]
\setlength{\belowcaptionskip}{5pt}
\centering\scriptsize
\caption{\textbf{Final inference regret at the fixed exploration exponent $\alpha=1/2$.} Standard denotes standard BO, while Uniform denotes BO with uniform random exploration. Table entries are based on 20 independent runs. Gains measure the improvement of Uniform over Standard; bold values indicate gains whose pointwise 95\% bootstrap confidence intervals exclude zero. Lower inference regret is better.}
\label{tab:app_synthetic_alpha_half}
\resizebox{\textwidth}{!}{%
\begin{tabular}{@{}llclll@{}}
\toprule
Objective & Dim. & Acquisition & Standard (mean $\pm$ SD) & Uniform (mean $\pm$ SD) & Gain [95\% CI]\\
\midrule
Trid & 10 & GP-UCB & 0.043 $\pm$ 0.021 & 0.028 $\pm$ 0.012 & \textbf{0.015 [0.005,0.027]}\\
Trid & 10 & LogEI & 0.044 $\pm$ 0.024 & 0.027 $\pm$ 0.009 & \textbf{0.018 [0.007,0.030]}\\
Trid & 10 & GP-TS & 0.019 $\pm$ 0.012 & 0.024 $\pm$ 0.012 & $-0.006$ [$-0.011$,0.000]\\
\addlinespace[1pt]
Rosenbrock & 20 & GP-UCB & 0.168 $\pm$ 0.132 & 0.126 $\pm$ 0.040 & 0.042 [$-0.010$,0.111]\\
Rosenbrock & 20 & LogEI & 0.162 $\pm$ 0.051 & 0.120 $\pm$ 0.027 & \textbf{0.041 [0.020,0.064]}\\
Rosenbrock & 20 & GP-TS & 0.075 $\pm$ 0.017 & 0.065 $\pm$ 0.009 & \textbf{0.010 [0.002,0.018]}\\
\addlinespace[1pt]
Shifted Bent Cigar & 30 & GP-UCB & 0.769 $\pm$ 0.370 & 0.462 $\pm$ 0.106 & \textbf{0.307 [0.145,0.492]}\\
Shifted Bent Cigar & 30 & LogEI & 0.514 $\pm$ 0.197 & 0.348 $\pm$ 0.094 & \textbf{0.166 [0.066,0.266]}\\
Shifted Bent Cigar & 30 & GP-TS & 0.127 $\pm$ 0.025 & 0.125 $\pm$ 0.032 & 0.003 [$-0.015$,0.020]\\
\addlinespace[1pt]
Sphere & 30 & GP-UCB & 0.902 $\pm$ 0.479 & 0.504 $\pm$ 0.136 & \textbf{0.398 [0.192,0.633]}\\
Sphere & 30 & LogEI & 0.485 $\pm$ 0.207 & 0.381 $\pm$ 0.152 & 0.104 [$-0.001$,0.212]\\
Sphere & 30 & GP-TS & 0.129 $\pm$ 0.036 & 0.106 $\pm$ 0.028 & \textbf{0.023 [0.007,0.041]}\\
\addlinespace[1pt]
Griewank & 30 & GP-UCB & 5.827 $\pm$ 1.922 & 2.590 $\pm$ 0.442 & \textbf{3.237 [2.415,4.126]}\\
Griewank & 30 & LogEI & 2.001 $\pm$ 0.807 & 1.896 $\pm$ 0.598 & 0.105 [$-0.358$,0.570]\\
Griewank & 30 & GP-TS & 0.501 $\pm$ 0.116 & 0.495 $\pm$ 0.125 & 0.006 [$-0.079$,0.086]\\
\addlinespace[1pt]
Sum of Different Powers & 30 & GP-UCB & 0.054 $\pm$ 0.122 & 0.006 $\pm$ 0.006 & \textbf{0.047 [0.005,0.108]}\\
Sum of Different Powers & 30 & LogEI & 0.024 $\pm$ 0.019 & 0.007 $\pm$ 0.009 & \textbf{0.018 [0.008,0.028]}\\
Sum of Different Powers & 30 & GP-TS & 0.002 $\pm$ 0.002 & 0.003 $\pm$ 0.003 & $-0.0004$ [$-0.0023$,0.0014]\\
\bottomrule
\end{tabular}}
\end{table*}
\FloatBarrier


\subsection{HPO acquisition coverage}
\label{app:exp_hpo}
\paragraph{Setup.} YAHPO Gym \citep{pfisterer2022yahpo} provides fixed surrogate benchmarks for reproducible hyperparameter optimization problems. We minimize validation log loss on the 14-dimensional XGBoost task \texttt{car} and the 28-dimensional \texttt{iaml\_super} task \texttt{sylvine}. The former tunes XGBoost hyperparameters, whereas the latter jointly selects a learning algorithm and its hyperparameters. The 14D and 28D experiments use total budgets of 120 and 200 evaluations, respectively, including 10 initial points, and each setting is evaluated over 30 independent runs. For both problems, we use the exploration schedule \(p_t=\min\{1,\log(t+1)/(t+1)^\alpha\}\) with \(\alpha=1/2\), and consider GP-UCB, LogEI, and GP-TS as the acquisition functions. For the exploration strategy, we consider both uniform random sampling and approximate greedy packing.


\paragraph{Results.} Figure~\ref{fig:app_hpo} reports improvements in final validation log loss over standard BO across the three acquisition functions, with positive values indicating better performance. On \texttt{car}, incorporating either uniform random exploration or greedy packing strategy yields statistically significant improvements for GP-UCB and LogEI. For GP-TS, uniform random exploration performs comparably to standard BO, while Greedy Packing slightly underperforms standard BO. On \texttt{sylvine}, uniform random exploration consistently improves GP-TS and LogEI, with no statistically significant difference for GP-UCB, whereas greedy packing yields statistically significant improvements for all three acquisition functions. Overall, these results demonstrate that incorporating an exploration strategy can improve the hyperparameter optimization performance of BO across a range of acquisition functions.

\begin{figure}[htb]
\centering
\includegraphics[width=\textwidth]{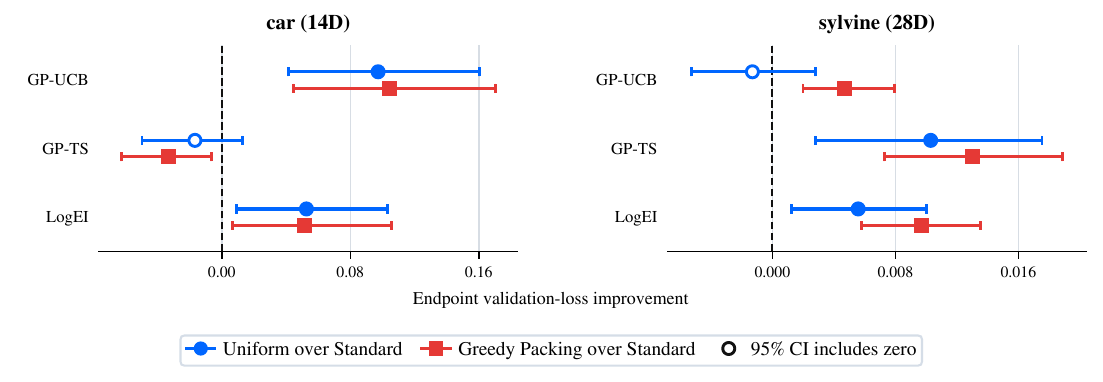}
\caption{\textbf{Final HPO comparisons across acquisition functions.}
Differences in final best-observed validation log loss between exploration-aided BO and standard BO (Standard) over 30 independent runs. Uniform denotes BO with uniform random exploration, while Greedy Packing denotes BO with approximate greedy packing exploration. Both exploration strategies use $\alpha=1/2$. Positive values indicate lower validation loss than Standard. Markers show mean improvements, and horizontal lines show pointwise 95\% bootstrap confidence intervals. Filled markers indicate intervals excluding zero.}
\label{fig:app_hpo}
\end{figure}

\FloatBarrier
\subsection{Lunar Lander controller tuning}
\label{app:exp_lunar}
\paragraph{Setup.} We consider a 12-parameter Lunar Lander controller-tuning problem \citep{eriksson2019scalable} over \([0,2]^{12}\). We evaluate each method over 30 independent runs. Each run uses a total budget of 400 evaluations, including 24 initial points. For a given random seed, all methods share the same initial design and the same set of 50 terrain seeds, yielding paired comparisons under identical initial conditions and evaluation environments. Each objective evaluation is the controller's mean return over these 50 terrains, which remain fixed throughout the run. Higher mean return indicates better controller performance. We regard a controller as tuned when its mean return over the 50 terrains is at least 200; this threshold is marked in Figure~\ref{fig:exp_applications}(c). We evaluate uniform random exploration with GP-UCB, LogEI, and GP-TS using the decaying schedule
$
p_t=\min\{1,\log(t+1)/(t+1)^{1/2}\}
$

\paragraph{Illustration on controller evolution.} 
Figure~\ref{fig:app_lunar_storyboard} illustrates the recommended controllers at four stages of standard BO (Standard) and BO with uniform random exploration using \(\alpha=1/2\) (Uniform), both with the LogEI acquisition function and the same random seed. Each panel shows the final scene from one representative terrain among the 50 terrains shared by the two methods. From 90 to 399 evaluations, Standard's mean return remains near 120, whereas Uniform's mean return increases from 154 to 285. This illustrative example shows how adding uniform random exploration can enable continued improvement when standard BO makes little further progress. 


\begin{figure*}[!t]
\centering
\includegraphics[width=0.93\textwidth]{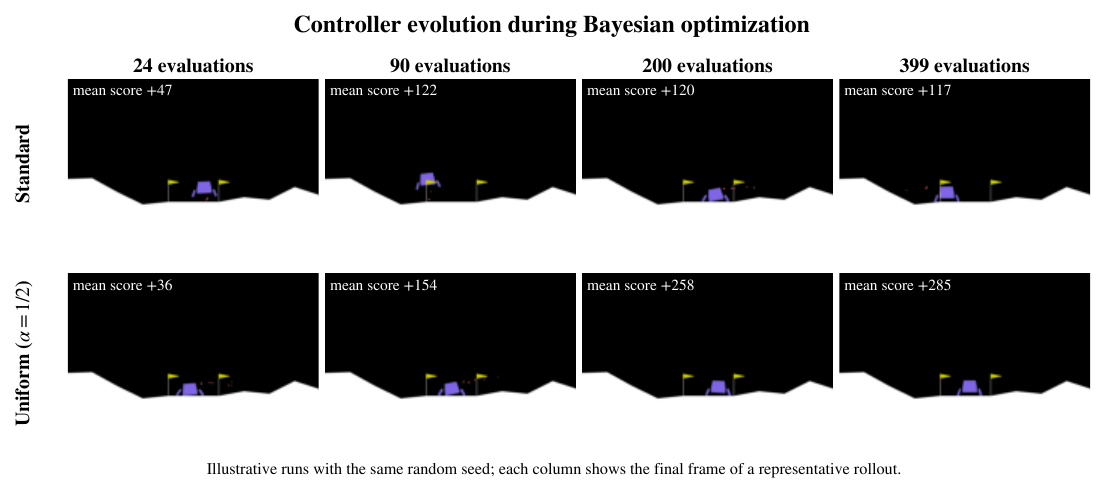}
\caption{\textbf{Evolution of a Lunar Lander controller during BO.} Recommended controllers from standard BO (Standard) and BO with uniform random exploration using \(\alpha=1/2\) (Uniform) at four stages during LogEI runs sharing the same random seed. Each panel shows the final frame on a representative terrain, and the reported score is the controller's mean return over the 50 terrains shared by the two methods.}
\label{fig:app_lunar_storyboard}
\end{figure*}

\FloatBarrier
\subsection{Computational comparison with MVR}
\label{app:exp_lowdim_mvr}
\paragraph{Setup.} We compare the computational cost of selecting exploration queries on two-dimensional shifted Ackley and shifted, rotated Rastrigin functions. All methods use GP-UCB with an exploration coefficient of two, two scrambled-Sobol initial observations, and 400 BO evaluations. We compare uniform random exploration with MVR-based exploration over 15 random seeds using the schedule \(p_t=\min\{1,\log(t+1)/(t+1)^{1/2}\}\). For each seed, the two methods start from the same initial observations, and a shared Bernoulli draw with probability \(p_t\) determines whether exploration occurs at each round. Thus, both methods explore at the same rounds but use different exploration rules: Uniform strategy samples a point uniformly at random from the search space, whereas MVR strategy selects the maximizer of the current posterior variance. On non-exploration rounds, both methods use GP-UCB based on their respective GP posteriors. Hence, the MVR-based method combines GP-UCB with scheduled variance maximization rather than maximizing posterior variance at every round. Standard BO uses GP-UCB at every round. All other experimental settings follow those of the synthetic benchmark experiments described above. Runtime is measured as the cumulative per-round wall-clock time for query selection, objective evaluation, GP fitting, and recommendation.

\paragraph{Inference regret.} Table~\ref{tab:app_lowdim_mvr} and Figure~\ref{fig:app_lowdim_mvr_runtime} summarize the results. Standard denotes standard BO using GP-UCB at every round, whereas Uniform denotes BO with scheduled uniform random exploration. Uniform achieves the lowest mean final inference regret on both objectives. Relative to Standard, the mean reduction in final inference regret is \(0.036\,[0.023,0.050]\) on Ackley and \(0.144\,[-0.391,0.674]\) on Rastrigin, where brackets denote 95\% bootstrap confidence intervals. On Ackley, MVR exhibits substantially greater variability,  where three high-regret runs between \(1.08\) and \(2.58\) inflate its mean regret. Its median regret is \(0.046\), compared with \(0.072\) for Standard, and MVR outperforms Standard in 9 of the 15 runs. 

\paragraph{Runtime comparison.} With \(\alpha=1/2\), the schedule yields 164 exploration rounds in expectation over 400 evaluations. On these rounds, uniform random exploration requires only a single uniform draw, whereas MVR performs multi-start posterior variance optimization from 20 starting points. Avoiding this additional optimization reduces mean runtime relative to MVR-based exploration by 44$\%$ on Ackley and 41$\%$ on Rastrigin. Relative to standard BO, which performs acquisition function maximization at every round, BO with uniform random exploration reduces mean runtime by 40$\%$ and 37$\%$, corresponding to average savings of 70 seconds and 54 seconds, with respective 95$\%$ confidence intervals of \([56,83]\) and \([44,63]\). Thus, BO with uniform random exploration is computationally more efficient than both standard BO and MVR-based exploration while also achieving the lowest mean final inference regret.

\begin{table}[htp]
\setlength{\belowcaptionskip}{5pt}
\centering\small
\caption{\textbf{Final inference regret and runtime on the two-dimensional benchmarks.} Entries report the mean (the standard deviation) over 15 random seeds. Runtime ratios are computed relative to standard BO on the same objective. Lower values indicate better performance for both metrics.}
\label{tab:app_lowdim_mvr}
\begin{tabular}{@{}llrrr@{}}
\toprule
Objective & Method & Final inference regret & Runtime (s) & Runtime / Standard\\
\midrule
Ackley 2D & Standard & 0.066 (0.025) & 174 (24) & 1.00\\
           & Uniform & \textbf{0.030 (0.020)} & \textbf{104 (14)} & \textbf{0.60}\\
           & MVR & 0.462 (0.902) & 186 (12) & 1.07\\
\addlinespace[1pt]
Rastrigin 2D & Standard & 0.684 (0.693) & 144 (13) & 1.00\\
              & Uniform & \textbf{0.540 (0.640)} & \textbf{90 (11)} & \textbf{0.63}\\
              & MVR & 0.755 (0.601) & 153 (6) & 1.06\\
\bottomrule
\end{tabular}
\end{table}

\begin{figure}[htp]
\centering
\includegraphics[width=\textwidth]{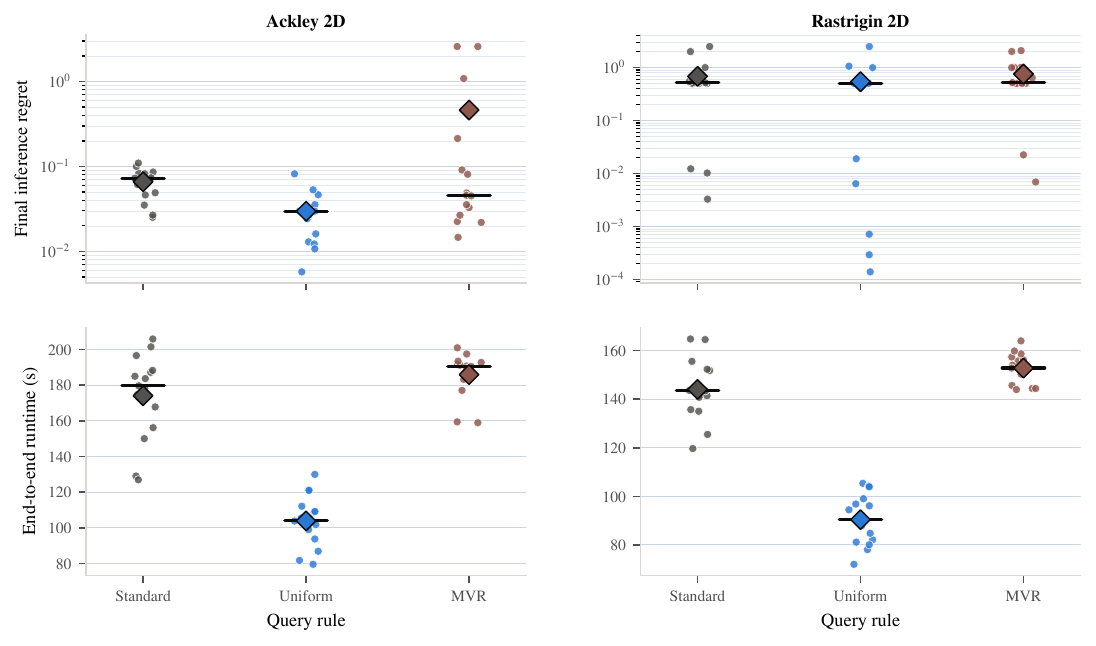}
\caption{\textbf{Final inference regret and runtime.} Final inference regret (top; logarithmic scale) and measured end-to-end optimizer runtime (bottom) over 15 random seeds. Points represent individual runs, horizontal segments indicate medians, and diamonds indicate means. }
\label{fig:app_lowdim_mvr_runtime}
\end{figure}
\FloatBarrier


\end{document}